\pdfoutput=1
\RequirePackage{ifpdf}
\ifpdf 
\documentclass[pdftex]{sigma}
\else
\documentclass{sigma}
\fi

\usepackage{overpic}

\numberwithin{equation}{section}

\def\C{\overline{\mathbb{C}}}
\def\R{\mathbb{R}}
\def\d{\mathrm{d}}
\def\Mob{{\text{\tiny M\"ob}}}

\def\F{\mathcal{F}}
\def\FCR{\mathrm{FCR}}
\def\CR{\mathrm{CR}}
\def\SCR{\mathrm{SCR}}

\newtheorem{Theorem}{Theorem}[section]
\newtheorem{Proposition}[Theorem]{Proposition}
 { \theoremstyle{definition}
\newtheorem{Definition}[Theorem]{Definition}
\newtheorem{Example}[Theorem]{Example}}

\begin{document}

\allowdisplaybreaks

\newcommand{\arXivNumber}{1603.09335}

\renewcommand{\PaperNumber}{080}

\FirstPageHeading

\ShortArticleName{M\"obius Invariants of Shapes and Images}

\ArticleName{M\"obius Invariants of Shapes and Images}

\Author{Stephen MARSLAND~$^\dag$ and Robert I.~MCLACHLAN~$^\ddag$}

\AuthorNameForHeading{S.~Marsland and R.I.~McLachlan}

\Address{$^\dag$~School of Engineering and Advanced Technology, Massey University,\\
\hphantom{$^\dag$}~Palmerston North, New Zealand}
\EmailD{\href{mailto:s.r.marsland@massey.ac.nz}{s.r.marsland@massey.ac.nz}}
\URLaddressD{\url{http://www.stephenmonika.net}}

\Address{$^\ddag$~Institute of Fundamental Sciences, Massey University, Palmerston North, New Zealand}
\EmailD{\href{mailto:r.mclachlan@massey.ac.nz}{r.mclachlan@massey.ac.nz}}
\URLaddressD{\url{http://www.massey.ac.nz/~rmclachl/}}

\ArticleDates{Received April 01, 2016, in f\/inal form August 08, 2016; Published online August 11, 2016}

\Abstract{Identifying when dif\/ferent images are of the same object despite changes caused by imaging technologies, or processes such as growth, has many applications in f\/ields such as computer vision and biological image analysis. One approach to this problem is to identify the group of possible transformations of the object and to f\/ind invariants to the action of that group, meaning that the object has the same values of the invariants despite the action of the group. In this paper we study the invariants of planar shapes and images under the M\"obius group $\mathrm{PSL}(2,\mathbb{C})$, which arises in the conformal camera model of vision and may also correspond to neurological aspects of vision, such as grouping of lines and circles. We survey properties of invariants that are important in applications, and the known M\"obius invariants, and then develop an algorithm by which shapes can be recognised that is M\"obius- and reparametrization-invariant, numerically stable, and robust to noise. We demonstrate the ef\/f\/icacy of this new invariant approach on sets of curves, and then develop a M\"obius-invariant signature of grey-scale images.}


\Classification{68T45; 68U10}

\section{Introduction}

Lie group methods play a fundamental role in many aspects of computer vision and image processing, including object recognition, pattern matching, feature detection, tracking, shape analysis, tomography, and geometric smoothing. We consider the setting in which a Lie group~$G$ acts on a space $M$ of objects such as points, curves, or images, and convenient methods of working in~$M/G$ are sought. One such method is based on the theory of invariants, i.e., on the theory of $G$-invariant functions on~$M$, which has been extensively developed from mathematical, computer science, and engineering points of view.

When $M$ is a set of planar objects the most-studied groups are the Euclidean, af\/f\/ine, similarity, and projective groups. In this paper we make a f\/irst study of invariants of planar objects under the M\"obius group $\mathrm{PSL}(2,\mathbb{C})$, which acts on the Riemann sphere
$\C = \mathbb{C}\cup\infty$ by
\begin{gather*}
 \phi\colon \ \C \to \C,\qquad \phi(z) = \frac{a z + b}{c z + d},\qquad a, b, c, d\in\mathbb{C}, \qquad
ad-bc\ne 0.
\end{gather*}
We work principally in the school of invariant signatures, developed by (amongst others) Olver and Shakiban
\cite{ames2002three,
calabi1998differential,
hoff2013extensions,
olver2001moving,
shakiban2004signature,
shakiban2005classification}, and widely used for Euclidean object recognition (see, e.g.,~\cite{aghayan2014planar}).

$\mathrm{PSL}(2,\mathbb{C})$ is a 6-dimensional real Lie group. It forms the identity component of the inversive group, which is the group generated by the M\"obius transformations and a ref\/lection. In addition to its importance on fundamental grounds~-- it is one of the very few classes of Lie groups that act on the plane, it crops up in numerous branches of geometry and analysis, it is the smallest nonlinear planar group that contains the direct similarities, and it is the set of biholomorphic maps of the Riemann sphere~-- it also has direct applications in image processing since it arises in the {\em conformal camera} model of vision, in which scenes are projected radially onto a sphere \cite{lenz1990group,turski2004geometric,turski2005geometric}. It may also correspond to neurological aspects of vision, such as grouping of lines and circles (which are equivalent under M\"obius transformations)~\cite{turski2006computational}.

From an applications point of view, dif\/ferent objects may be related by M\"obius transformations, as is explored by
Petukhov \cite{petukhov1989non} for biological objects in fascinating detail in the context of Klein's Erlangen program\footnote{D'Arcy Wentworth Thompson \cite{thompson1942growth} famously deformed images of one species to match those of another; his theory of transformations is reviewed and interpreted in light of modern biology in \cite{arthur2006d}. In particular, it has given rise to image processing techniques such as {\em Large Deformation Diffeomorphic Metric Mapping $($LDDMM$)$} \cite{glaunes2008large}, in which images are compared modulo {\em infinite}-dimensional groups such as the dif\/feomorphism group. Petukhov writes that Thompson ``did not use the Erlanger program as the basis in this comparative analysis.'' However, the totality of Thompson's examples and the explanations in his text do indicate that in all cases he selected his transformations from the simplest group that would do an (in his view) acceptable job. Eight dif\/ferent groups are identif\/ied in \cite[Table~1]{marsland2013geodesic}; four are f\/inite-dimensional. Thus, whether consciously or not, Thompson's work was fully consistent with the Erlangen program. Of relevance to the present paper is that many of his examples use conformal mappings, and thus may be approximated (or even determined by) M\"obius transformations.}, giving examples of many body parts that are loxodromic to high accuracy (loxodromes have constant M\"obius curvature and play the role in M\"obius geometry that circles do in Euclidean geometry), that change shape by M\"obius and other Lie group actions, and that grow via M\"obius transformations (including 2D representations of the human skull). Other examples of 1-dimensional growth patterns, such as antenatal and postnatal human growth, appear to be well modelled by 1-dimensional linear-fractional transformations $x\mapsto (ax+b)/(cx+d)$. Discussing this work, Milnor~\cite{milnor2010growth} argued that ``The geometrically simplest way to change the relative size of dif\/ferent body parts would be by a~conformal transformation. It seems plausible that this simplest solution will often be the most ef\/f\/icient, so that natural selection tend to choose it''. (Milnor was thinking of~3D conformal transformations, whose restriction to~2D is the M\"obius transformations.)

\looseness=-1 In this paper we present an integral invariant for the 2D M\"{o}bius group that is reparameterization independent, and demonstrate its use to identify curves that are related by M\"{o}bius transformations. We then consider the case of images, and describe an invariant signature by which M\"{o}bius transformed images can be recognised. We begin in Section \ref{sec:invariants} by discussing the desirable properties of invariants for object classif\/ication, and introduce the property of bounded distortion, which subsumes most of the requirements identif\/ied in the literature. This is followed by an overview of the methods that have been used to give object classif\/ication invariants for curves and images (most often with respect to the Euclidean, similarity, and af\/f\/ine groups).

In Section \ref{sec:eg} we present the classical M\"obius invariants and discuss their utility with respect to the properties identif\/ied in Section~\ref{sec:invariants}, before using a numerical example based on an ellipse to demonstrate the dif\/ference between them. This is followed by the introduction of the invariant that we have identif\/ied as the best behaved for curves with respect to the requirements previously discussed. In order to demonstrate its utility, we present an experiment where a set of smooth Jordan curves are created, and then the invariant distance is computed, and compared with direct registration of each pair of curves in the M\"{o}bius group using the $H^1$ similarity metric. The results show that the invariant is well-behaved with respect to noise, and can be used to separate all but the most similar shapes.

In Section \ref{sec:images} we move on to images and demonstrate the use of a~3D M\"{o}bius signature that is very sensitive and relatively cheap to compute, while still being more robust than the analogous signature for curves. As far as we know, such dif\/ferential invariant signatures (in the sense of Olver~\cite{olver2001moving}) for images have not been considered previously.

\section{Invariants of shapes and images}\label{sec:invariants}
\subsection{Computational requirements of invariants for object classif\/ication}\label{sec:req}

The mathematical def\/inition of an invariant, namely, a $G$-invariant function on $M$, is not suf\/f\/iciently strong for many computational applications. For example, object classif\/ication via invariants involves comparing the values of the invariant on dif\/ferent $G$-orbits, and the def\/inition says nothing about this. Recognising this, Ghorbel~\cite{ghorbel1994complete} has given the following partial list of qualities needed for object recognition:
\begin{itemize}\itemsep=0pt
\item[(i)] fast computation;
\item[(ii)] good numerical approximation;
\item[(iii)] powerful discrimination (if two objects are far apart modulo $G$, their invariants should be far apart);
\item[(iv)] completeness (two objects should have the same invariants if\/f they are the same modulo~$G$);
\item[(v)] provide a $G$-invariant distance on $M$; and
\item[(vi)] stability (if two objects have nearby invariants, they should be nearby modulo~$G$).
\end{itemize}

Calabi et al.~\cite{calabi1998differential} include in (ii) the requirement that the numerical approximation itself should be $G$-invariant, while Abu-Mostafa et al.~\cite{abu1984recognitive} add the further desirable qualities of:
\begin{itemize}\itemsep=0pt
\item[(vii)] robustness (if two objects are nearby modulo $G$, especially when one is a noisy version of the other, their invariants should be nearby);
\item[(viii)] lack of redundancy (i.e., all invariants are independent); and
\item[(ix)] lack of suppression (in which the invariants are insensitive to some features of the objects).
\end{itemize}

Manay \cite{manay2006integral} includes, in addition:
\begin{itemize}\itemsep=0pt
\item[(x)] locality (which allows matching subparts and matching under occlusion).
\end{itemize}

To this already rather demanding list we add one more, that the set of invariants should be:
\begin{itemize}\itemsep=0pt
\item[(xi)] small.
\end{itemize}

\begin{figure}[t]\centering
\includegraphics[width=11.5cm]{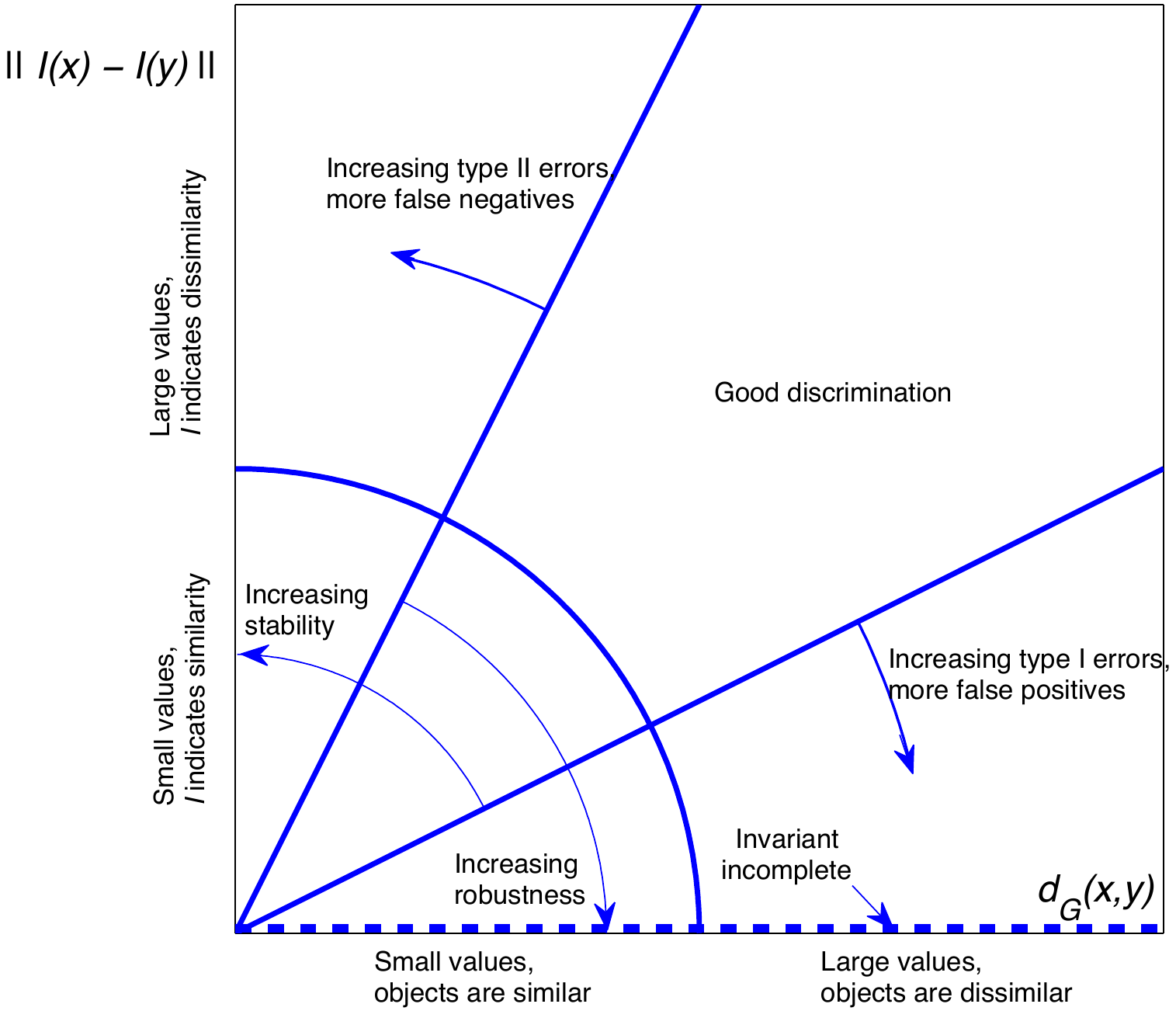}
\caption{Sketch of the approximate relationships between $d_G(x,y)$, which measures the distance between two objects $x$ and $y$ modulo a symmetry group $G$, and $\|I(x)-I(y)\|$, which measures the distance between their invariants, illustrating the role of the desired properties of completeness, robustness, stability, and discrimination. In object recognition, `good discrimination' is sometimes taken to mean the avoidance of type I errors (relative to the null hypothesis that two objects are in the same group orbit, so that in a type~I error dissimilar objects are classif\/ied as similar), the avoidance of type~II errors (similar objects being classif\/ied as dissimilar), or is not specif\/ied; we take it to mean that both type~I and type~II errors are avoided.}\label{fig:diagram}
\end{figure}

The motivation for this last is purely a parsimony argument; it is easy to produce large sets of invariants without adding much utility, and smaller sets of easily computable invariants are to be preferred. This is particularly the case since some of these criteria are in conf\/lict, so there will usually not be one invariant that is preferred for all applications; instead, the best choice will depend on the dataset and the particular application. Others are closely related, especially discrimination, distance, stability, robustness, and suppression; these all depend on which features of the objects are deemed to be signal and which are deemed to be noise.

These criteria can be unif\/ied and quantif\/ied in the following way: suppose that a distance on objects has been chosen that measures the features that we are interested in (the signal), and that is small for dif\/ferences that we are not interested in (the noise). This distance induces a~distance on objects modulo~$G$ by (where $\| \cdot \|_d$ is some appropriately chosen norm):
\begin{gather*}
 d_G(x, y) := \max\big(\inf_{g\in G} \|g\cdot x - y\|_{d},\, \inf_{g\in G}\|g\cdot y - x\|_{d}\big).
 \end{gather*}
(In principle one can compute $d_G(x,y)$ by optimizing over $G$, and indeed this is done in many applications. However, optimization is relatively dif\/f\/icult and unreliable, and becomes impractical on large sets of objects; this is the step that invariants are intended to avoid.) Suppose that a norm on invariants $I\colon M\to \mathbb{R}^k$ has been chosen.
Then we can express the criteria as follows:
\begin{itemize}\itemsep=0pt
\item[(iii)] discrimination: $d_G(x,y)\gg \varepsilon \Leftrightarrow \|I(x)-I(y)\| \gg \varepsilon$,
\item[(iv)] completeness: $\|I(x)-I(y)\| = 0 \Rightarrow d_G(x,y)=0$,
\item[(vi)] stability: $\|I(x)-I(y)\| \lesssim \varepsilon \Rightarrow d_G(x,y) \lesssim \varepsilon$,
\item[(vii)] robustness: $d_G(x,y) \lesssim \varepsilon \Rightarrow \| I(x)-I(y) \| \lesssim \varepsilon$.
\end{itemize}

These may be subsumed and strengthened by the property of {\em bounded distortion}: an invariant has bounded distortion with respect to $d_G$ if there exist positive constants $c_1$, $c_2$ such that for all $x,y\in M$ we have
\begin{gather}\label{eq:boundeddistortion}
c_1 d_G(x,y) \le \| I(x) - I(y) \| \le c_2 d_G(x,y).
\end{gather}

Suppose the invariant is to be used to test the hypothesis that two objects are the same modulo symmetry group~$G$.
The situation is illustrated in Fig.~\ref{fig:diagram}, which plots the distance between the invariants of two objects against the distance between the objects modulo~$G$. The smaller the value of $c_2$, the more robust the invariant is, and the fewer false positives will be reported. The larger the value of~$c_1$, the more stable the invariant is, the better its discrimination, and the fewer false negatives will be reported. If~(\ref{eq:boundeddistortion}) holds only with $c_1=0$, the invariant is neither stable nor complete; if there is no~$c_2$ such that~(\ref{eq:boundeddistortion}) holds, the invariant is not robust. In addition, in some applications one may be more interested in some parts of the space shown in Fig.~\ref{fig:diagram} than others: for example, in discriminating very similar or very dissimilar objects.

In practice, bounded distortion may not be possible. Even completeness, the centrepiece of the mathematical theory of invariants, can be very demanding. It turns out that many invariants that are roughly described as `complete' are not fully complete, that is, they do not distinguish {\em all} group orbits, but only distinguish {\em almost all} group orbits. For example, Ghorbel's Euclidean invariants of grey-level images~\cite{ghorbel1994complete} only distinguish images in which two particular Fourier coef\/f\/icients are non-zero. Thus, it will be stable only on images in which these two Fourier coef\/f\/icients are bounded away from zero. The most commonly-used Euclidean signature of curves, $(\kappa,\kappa_s)$, is complete only on nondegenerate curves~\cite{hickman2012euclidean}. The fundamental issue is that even for very simple group actions, the space of orbits can be very complicated topologically, and thus very dif\/f\/icult (or even impossible) to coordinatize via invariants. We illustrate these ideas via an example:

\begin{Example}\label{ex:npts} Consider $n$ points $z_1,\dots,z_n$ in the complex plane. Let $G=S^1$ act by rotating the points, i.e., $e^{i\theta}\cdot(z_1,\dots,z_n) = (e^{i\theta}z_1,\dots,e^{i\theta} z_n)$. Then $\{\bar z_i z_j\colon 1\le i\le j \le n\}$ forms a~complete set of invariants. It is very large (there are $n^2$ real components) compared to the dimension $2n-1$ of the space, $\mathbb{C}^n / S^1$, in which we are trying to work. However, if any one real component of the set is omitted, the resulting set is not complete. For example, if $|z_1|^2$ is omitted, then the points $(1,0,\dots,0)$ and $(2,0,\dots,0)$ have the same invariants, but do not lie on the same orbit. This is the situation considered in the problem of {\em phase retrieval}~\cite{bandeira2014saving}. By choosing combinations of $\{\bar z_i z_j\}$ it is possible to create smaller sets of complete invariants, but the computational complexity of calculating the description of the $n$ points is still $\mathcal{O}(n^2)$~\cite{bandeira2014saving}.
\end{Example}

In cases where one settles for an incomplete set of invariants, or a set that is not of bounded distortion~-- so that the worst-case behaviour of the invariants is arbitrarily poor~-- it can make sense instead to study the average-case behaviour of the invariants over some distribution of the objects. This is analogous to the role that ill-posed and ill-conditioned problems play in numerical linear algebra, although in that case the ill-conditioning is intrinsic rather than imposed by considerations of computational complexity. Indeed, one can sometimes very usefully describe objects based on an extremely small number of invariants, for example, describing planar curves by their length and enclosed area, or by the f\/irst few Fourier moments of their curvature. What is wanted is to optimize the behaviour of the invariants, over some distribution of the objects, with respect to their time and/or space complexity. However, we know of no genuine cases in which such a program has been carried out.

\subsection{Invariants of curves}\label{sec:invariants2}

There is a large literature concerning invariants to the Euclidean and similarity groups for both curves and images. The purpose of this section is to provide an overview of the dominant themes within that work that are relevant to our goal of identifying invariants for the M\"{o}bius group that satisfy at least some of the criteria listed in the previous section.

We def\/ine a (closed) shape as the image of a function $\phi\colon S^1\to\R^2$. For dif\/ferent restrictions on $\phi$ this def\/ines dif\/ferent spaces: if $\phi$ is a continuous injective mapping then this def\/ines simple closed curves, while if $\phi$ is dif\/ferentiable and $\phi'(t)\ne 0$ for all $t$ this def\/ines the `shape space' $\operatorname{Imm}(S^1,\R^2)/\operatorname{Dif\/f}(S^1)$~\cite{Mumford98}, while if $\phi$ is an immersion and $\phi(S^1)$ is dif\/feomorphic to $S^1$, we obtain the shape space $\operatorname{Emb}(S^1,\R^2)/\operatorname{Dif\/f}(S^1)$ (roughly, the curve has no self-intersections). There are also smooth ($C^\infty$) versions of these, piecewise smooth versions, shapes of bounded variation, and so on. The geometry of these shape spaces is now studied intensively in its own right~\cite{michor1980manifolds,michor2007overview} as well as as a setting for computer vision; see~\cite{bauer2013overview} for a recent review.

The quotient by $\operatorname{Dif\/f}(S^1)$ in the shape spaces has the ef\/fect of factoring out the dependence on the parameterization. Various methods have been used to achieve this, including:
\begin{enumerate}\itemsep=0pt
\item Using a standard parameterization, such as Euclidean arclength. This leaves a dependence on the choice of starting point, i.e., the subgroup of $\operatorname{Dif\/f}(S^1)$ consisting of translations is not removed.
\item Moment invariants, $\int_0^L f(\phi(s))\, \d s$, where $L$ is the length of the curve, $s$ is Euclidean arclength and $f$ ranges over a set of basis functions, such as monomials or Fourier modes~\cite{abu1984recognitive}.
\item Currents $\int_{S^1} \phi^*\alpha$, where $\alpha$ ranges over a set of basis 1-forms on $\R^2$, such as monomials times $\d x$ and monomials times~$\d y$~\cite{glaunes2008large}.
\end{enumerate}

The next step is to consider a Lie group $G$ acting on $\R^2$ that induces an action on the shape space.
Here, many methods and types of invariants have been investigated (see, e.g.,~\cite{van1995vision}).
\begin{description}\itemsep=0pt
\item[The moving frame method] is a general approach to constructing invariants. Objects are put into a reference conf\/iguration and their resulting coordinates are then invariant. The mo\-ving frame or reference conf\/iguration method was developed as a way of {\em finding} dif\/ferential invariants and variants of them (see~\cite{Olver2005} for an overview). A~simple application of the method is the situation considered in Example~\ref{ex:npts}, of $n$ points in the plane under rotations. Consider conf\/igurations with $z_1\ne 0$, $|\arg z_1|<\pi$. Rotate the conf\/iguration so that~$z_1$ lies on the positive real axis. The coordinates of the resulting reference conf\/iguration, $\bar z_1 z_j / |z_1|^2$, are invariant. In this case, the invariants are well behaved as~$|\arg z_1|\to \pi$, but not as $z_1\to 0$.
\item[Joint invariants] are functions of several points; for example, pairwise distances for a complete set of joint invariants for the action of the Euclidean group on sets of points in the plane.
\item[Dif\/ferential invariants] are functions of the derivatives of a curve at a point. There exist algorithms to generate all dif\/ferential invariants~\cite{Olver2005}. For the action of the Euclidean group on planar curves, the Euclidean curvature $\kappa = \phi'\times \phi''/\|\phi'\|^3$ is $E(n)$ invariant (and parameterization invariant), and its derivatives $d^n\kappa/ds^n$ with respect to arclength form a~complete set of dif\/ferential invariants.
\item[Semi-dif\/ferential invariants] (also known as joint dif\/ferential invariants~\cite{olver01joint}) of a curve are functions of several points and derivatives.
\item[Integral invariants] are formed from the moments or the partial moments $\int_{s_0}^s f(\phi(t))\, \d t$. With some care they can be made parameterization- and basepoint-independent \cite{feng2010classification,manay2006integral}. Initially, these invariants appear to have some advantages, being relatively robust and often including some locality. However, they are not always applicable; for example \cite{manay2006integral}, which used regional integral invariants, still requires a point correspondence optimization in order to get a distance between shapes. In addition, groups such as the projective group do not act on any f\/inite subset of the moments \cite{van1995vision}. Astrom \cite{aastrom1994fundamental} shows that there are no stable projective invariants for closed planar curves. While local integral invariants are promising, as reported by~\cite{hann2002projective}, there may be analytic dif\/f\/iculties in deriving them.
\end{description}

Another example of the moving frame method, which is common in image processing and shape analysis, is centre of mass reduction. The centre of mass of a shape may be moved to the origin in order to remove the translations. This may be calculated by, for example, $\int_{\phi(S^1)}(x^2/2,x y) \d y = \iint_{\mathrm{int}\phi(S^1)}(x,y) \d x \d y$. However, the shapes $x = a + \sin t$, $y=0$ have centre of mass equal to 0 for all values of $a$, even though they are related by translations. Thus, this method would not be robust on any dataset that contained shapes approaching such degenerate shapes. For rotations, the reference conf\/iguration method will not be robust if there are objects close to having a discrete rotational symmetry. The underlying problem with the reference conf\/iguration approach is that it is attempting to use a set of invariants equal to the dimension of the desired quotient space $M/G$. This space is almost always non-Euclidean, so such a set can only be found on some subset of $M/G$ of Euclidean topology. The set is then robust only on datasets that are bounded away from the boundary of this subset.

\looseness=-1 Calabi et al.~\cite{calabi1998differential} propose the use of dif\/ferential invariant signatures for shape analysis, and further argue that these should be approximated in a group-invariant way. For example, for the Euclidean group, the signature is the shape $(\kappa,\kappa_s)$ regarded as a subset of $\R^2$. The claimed advantages of the approach are that the signature determines the shape; that it does not depend on the choice of initial point on the curve or on parameterization by arclength, and that its $G$-invariance makes it robust; and that it is based on a general procedure for arbitrary objects and groups. See \cite{aghayan2014planar, ames2002three,
calabi1998differential,
hoff2013extensions,
olver2001moving,
shakiban2004signature,
shakiban2005classification} for further developments and applications of the invariant signature.

Although the signature at f\/irst sight appears to be complete (e.g., Theorem~5.2 in \cite{calabi1998differential}), a~more detailed treatment (e.g., \cite{hickman2012euclidean, olver2000moving}) highlights the fact that it is not complete on shapes that contain singular parts~-- straight and circular segments in the Euclidean case. For these parts, the signature reduces to a point. Thus, for shapes that have {\em nearly} straight or nearly circular segments, the signature cannot be robust. In addition, while the signature does not depend on the starting point or the parameterization, it takes values in a very complicated set, namely, the planar shapes. To compare the invariants of two shapes requires comparing two shapes. Essentially, the parameterization-dependence has only been deferred to a later stage of the analysis (unless one is content to compare shapes visually). One approach to this is to weight the signature, see~\cite{olver15groupoid} for more details.

\looseness=-1 The claim in \cite{calabi1998differential} that the method's $G$-invariance makes it robust should be assessed through further analysis and experiment. It would appear to be most relevant for datasets in which the errors due to the presentations of the shapes are comparable to those resulting from the errors in the shapes themselves (i.e., noise) and from the distribution of the objects in a classif\/ication problem.
Their f\/inal point, that the method is extremely general, is a powerful one. Calabi~\cite{calabi1998differential} carry out the procedure for the Euclidean and for the 2D af\/f\/ine group. However in Section~6.7 they report that ``the interpolation equations in general are transcendentally nonlinear and do not admit a readily explicit solution'', indicating that the method may not succeed for all group actions.

It is also possible to represent the shape as a binary image and apply image-based invariants (which are described next). This has the advantage of working directly in the space $\R^2$ on which the group acts, and avoiding all questions of parameterization, etc., but it does sacrif\/ice a lot of information about the shape. A related approach, which is popular in PDE-based method for curves, is to represent the shape as the level set of a smooth function $\phi\colon \R^2\to\R$. This is also parameterization-independent, and retains smoothness, but we have never seen it used in for constructing invariants.

\subsection{Invariants of images}\label{sec:iminv}

Most studies of image invariants has been based on grey-level images $f\colon\Omega\to[0,1]$, where $\Omega\subset\R^2$. The methods are primarily based on moments \cite{abu1984recognitive} or Fourier transforms \cite{ghorbel1994complete,kakarala2012bispectrum} of the images, and have been highly developed for the translation, Euclidean, and similarity groups, where the linearity of the action and the special structure of these Lie groups means that the approach is particularly fast and robust. Attempts have been made to extend the method of Fourier invariants to other groups. There is an harmonic analysis for many non-Abelian Lie groups, including, in fact, the M\"obius group \cite{taylor1990noncommutative}, as well as a general theory for compact non-commutative groups~\cite{Gauthier2008}. There are some applications of this theory to image processing \cite{kakarala2012bispectrum,turski2006computational} and to other problems in computational science. Fridman \cite{fridman2000numerical} discusses a Fourier transform for the hyperbolic group, the 3-dimensional subgroup of the M\"obius group that f\/ixes the unit disc. However, the theory appears not to have been developed to the point where it can be used as ef\/fectively as the standard Fourier invariants. Therefore, in this paper, in Section \ref{sec:images}, we develop a M\"obius invariant of images, based on a~dif\/ferential invariant signature of the image.

\section{M\"{o}bius invariants for curves}\label{sec:eg}

In this section we consider invariants of curves for the M\"{o}bius group, and derive one suitable for practical computation, demonstrating its application for a set of simple closed curves.

\subsection{Known M\"obius invariants}

The most well-known invariant of the M\"{o}bius group is the cross-ratio, also known as the {\em wurf}, which is based on the ratio of the distances between a set of 4 points:
\begin{gather*}
\CR(z_1,z_2,z_3,z_4) := \frac{(z_1-z_3)(z_2-z_4)}{(z_2-z_3)(z_1-z_4)},
\end{gather*}
where the invariance means that $\CR(Tz_1,Tz_2,Tz_3,Tz_4) = \CR(z_1,z_2,z_3,z_4)$ for M\"{o}bius transformation $T$.

Since the M\"obius group can send any triple of distinct points in $\C$ into any other such triple, there are no joint invariants of~2 or~3 points; the orbits are the conf\/igurations consisting of~1,~2, and~3 distinct points. When $n>4$, the set of cross-ratios of any four of the points forms a~complete invariant.

For large numbers of points, this set of all cross-ratios has cardinality $\mathcal{O}(n^4)$ which is impractically large (although some may be eliminated using functional relations amongst the invariants, known as syzygies \cite{olver01joint}). However, if the dataset of shapes or images is tagged with a small number of clearly-def\/ined landmarks, then some subset of the set of cross-ratios may form a useful invariant. This is the method by which Petukhov \cite{petukhov1989non} was able to identify linear-fractional, M\"obius, and projective relationships in biological shapes. For untagged objects, automatic tagging may be possible using critical points (e.g., maxima and minima) of images, and their values; these are homomorphism- and hence M\"obius-invariant, and can be identif\/ied, even in the pre\-sence of noise, by the method of persistent homology~\cite{edelsbrunner2008persistent}. However, such invariants are clearly highly incomplete for shapes and images, and we do not study them further here.

In order to derive dif\/ferential invariants for the M\"obius group, the most useful starting point is the Schwarzian derivative
\begin{gather}\label{eq:schwarzian}
(Sz)(t) = \left( \frac{z''}{z'} \right)' - \frac{1}{2} \left( \frac{z''}{z'} \right)^2 = \frac{z'''}{z'} - \frac{3}{2} \left( \frac{z''}{z'} \right)^2.
\end{gather}
By an abuse of notation, which is standard in the literature (see, for example, the very readable~\cite{ahlfors1988cross}), the same formula~\eqref{eq:schwarzian} is used in three dif\/ferent situations: when $z\colon \R\to \R$ (used in studying linear-fractional mappings in real projective geometry); when $z\colon \C\to \C$ (used in stu\-dying complex analytic mappings); and when $z\colon M\to\C$, $M$ a real 1-dimensional manifold (used in studying the M\"obius geometry of curves). We adopt the latter setting so that~$z'$ is the tangent to the curve. The Schwarzian derivative is then invariant under M\"obius transformations~$\phi$:
\begin{gather*}
S(\phi\circ z)(t) = S(\phi)(t)
\end{gather*}
and under reparameterizations $\psi\colon M\to M$ transforms as
\begin{gather*}
S (z \circ \psi) = (S (z) \circ \psi) \cdot (\psi')^2 + S (\psi),
\end{gather*}
where the last term is the {\em real} Schwarzian derivative. Therefore (where $\operatorname{Im}$ represents the imaginary part of a complex number)
\begin{gather*}
\operatorname{Im} S (z \circ \psi) = (\operatorname{Im} S (z) \circ \psi) \cdot (\psi')^2.
\end{gather*}
This can be used to construct a distinguished M\"obius-invariant parameterization of the curve. Let $\tilde z(\lambda) = z(t)$ where $\lambda=\psi(t)$. The parameter $\lambda$ will be chosen so that $\operatorname{Im} S(\tilde z) \equiv 1$. This gives
\begin{gather*}
\operatorname{Im} S(z) = \operatorname{Im} S(\tilde z\circ \psi) = (\operatorname{Im} S(\tilde z)\circ \psi) \cdot (\psi')^2 = (\psi')^2.
\end{gather*}
The choice $\psi'(t) = \sqrt{| \operatorname{Im} S(z)(t)|}$ achieves this while preserving the sense of the curve, while the choice $\psi'(t) = -\sqrt{| \operatorname{Im} S(z)(t)|}$ achieves it while reversing the sense. Put another way, the parameter
\begin{gather*}
\lambda = \int_{t_0}^t \sqrt{| \operatorname{Im} S(z)(t)|}
\end{gather*}
is invariant under M\"obius transformations and {\em almost} invariant under sense-preserving reparameterizations of $t$, which act as translations in $\lambda$ (because of the freedom to choose $t_0$). Equivalently, the 1-form
\begin{gather*}
\d \lambda = \sqrt{| \operatorname{Im} S(z)(t)|} \d t,
\end{gather*}
known as the M\"obius or inversive arclength, is M\"obius- and sense-preserving parameterization-invariant.

This is often stated in a form (originally due, according to Ahlfors \cite{ahlfors1988cross}, to Georg Pick) using the Euclidean curvature $\kappa$. If the curve is parameterized by Euclidean arclength $s$, then its tangent $\theta(s) = z'(s)/\|z'(s)\|$ and $\kappa(s) = \theta'(s)$. Dif\/ferentiating again leads to $\operatorname{Im} S(z)(s) = \kappa'(s)$, or
\begin{gather*}
\d\lambda = \sqrt{| \kappa'(s)| } {\rm d}s.
\end{gather*}

The 1-form $\d\lambda$ provides a useful discrete invariant, the M\"obius length $L$ of the curve
\begin{gather*}
L = \int_M \d\lambda.
\end{gather*}

The real part of $S(\tilde z)(\lambda)$ is now a parameterization-invariant M\"{o}bius invariant known as the inversive or M\"{o}bius curvature~\cite{Patterson1928}
\begin{gather*}
\kappa_{\Mob} = \frac{4\kappa' (\kappa''' -\kappa^2 \kappa') - 5 (\kappa'')^2 }{8(\kappa')^3} ,
\end{gather*}
where $'$ denotes dif\/ferentiation with respect to arclength $s$.

The set of all dif\/ferential M\"{o}bius shape invariants is then $\d^n\kappa_{\Mob}/ d\lambda^n$, $n\ge 0$. Following Calabi et al.~\cite{calabi1998differential}, two possible candidate invariants that could be used to recognize M\"obius shapes are the function $\kappa_\Mob(\lambda)$ modulo translations and the signature $(\kappa_\Mob,\d\kappa_\Mob/\d\lambda)(S^1)\subset\R^2$. These are complete on sections of shapes with no vertices (points where $\kappa'(s)=0$). However, since $\kappa_\Mob$ requires the 5th derivative of the curve (third derivative of the curvature), it is not robust in the presence of noise, and we do not explore it further.

There are also invariants based on forms of higher degree. We give just one example, the {\em M\"obius energy}
\begin{gather}\label{eq:energy}
-\frac{1}{2}\iint_{M \times M} \frac{\sin\theta_u \sin\theta_v}{|v-u|^2}\d u\, \d v,
\end{gather}
introduced by O'Hara and Solanes \cite{o2010m}, see also the related Kerzman--Stein distance \cite{barrett2007cauchy}. Here $\d u$, $\d v$ are Euclidean arclength and $\theta_u$ (resp. $\theta_v$) is the angle between $v-u$ and $z'(u)$ (resp.~$v$). It represents a renormalization of the energy of particles on the curve, distributed evenly with respect to Euclidean arclength and interacting under an $r^{-4}$ potential. (The authors of \cite{o2010m} comment that ``due to divergence problems, almost nothing is known about integral geometry [i.e., about invariant dif\/ferential forms] under the M\"obius group''.) The M\"{o}bius energy has one distinguishing feature compared to the other invariants: its def\/inition depends only on the f\/irst derivative of the curve. The singularity at $u=v$ is removable, since the integrand obeys
\begin{gather*}
\frac{\sin\theta_u \sin\theta_v}{|v-u|^2} = \frac{1}{4}\kappa(u)\kappa(v)+\mathcal{O}(|v-u|).
\end{gather*}
Thus, the energy is def\/ined for $C^2$ curves, and even near the singularity it depends only on the curvature of the curve.

However, in numerical experiments we found the integrand of~\eqref{eq:energy} dif\/f\/icult to evaluate accurately and invariantly, particularly near the diagonal, while the double integral generated little improvement in robustness. The most negative feature of the energy is its cost: its evaluation apparently requires considering all pairs of points on a discrete curve. There is another reason why invariant 2-forms are less useful than invariant 1-forms like $\d\lambda$: invariant $k$-forms distinguish coordinates on $M^k$ only up to $k$-form- (i.e., volume-) preserving maps. These are inf\/inite-dimensional for $k>1$, but 1-dimensional for $k=1$: the coordinates are determined up to translations. For these reasons, we do not consider the energy, or related invariant 2-forms, any further.

Finally, M\"obius transformations map circles to circles, and thus critical points of Euclidean curvature (the `vertices' of the shape) are M\"obius invariants. The four vertex theorem states that all smooth curves have at least four vertices. The number of vertices is a discrete M\"obius invariant. Consequently, sections of shapes on which the Euclidean curvature is monotonic, and their M\"obius lengths, are M\"obius invariant.

\subsection{Example: Evaluating the M\"obius length of an ellipse}\label{sec:ellipse}

Having described a number of possible invariants, and ruled many of them out based on the criteria discussed in Section~\ref{sec:req}, we now provide a concrete example, illustrating and testing some of these constructions on the ellipse $z(t) = \cos 2 \pi t + 2 {\rm i} \sin 2 \pi t$, $0\le t\le 1$. The ellipse is discretized at $t_i = (i+1/4) h$, $i=0,\dots,n$, giving points $z_i = z(t_i)$, and the M\"obius arclength $\d\lambda$ is calculated in two ways:
\begin{itemize}\itemsep=0pt
\item From the {\em curvature} method, in which the Euclidean curvature $\kappa$ is calculated in a Eucli\-dean-invariant way\footnote{Following Calabi et al.~\cite{calabi1998differential}, let~$A$,~$B$,~$C$ be three points on a curve, and let $a=d(A,B)$, $b=d(B,C)$, and $c=d(C,A)$ be the Euclidean distances between them. Then the curvature of the circle interpolating~$A$,~$B$, and~$C$ is
\begin{gather*}
\kappa(A,B,C) = \pm 4\frac{\sqrt{s(s-a)(s-b)(s-c)}}{abc}.
\end{gather*}
}
by interpolating a circle through 3 adjacent points, and $\d\kappa/\d s$ by a~f\/inite dif\/ference, giving
\begin{gather}\label{eq:method1}
\d\lambda(t_{i+3/2}) \approx | \kappa(z_{i+1},z_{i+2},z_{i+3}) - \kappa(z_{i},z_{i+1},z_{i+2}) | (| z_{i+2}-z_{i+1}| ).
\end{gather}
\item From the {\em cross-ratio} method, using
\begin{gather}\label{eq:method2}
 \d\lambda(t_{i+3/2}) \approx \sqrt{ 6 | \mathrm{Im}(\log(\CR(z_i,z_{i+1},z_{i+2},z_{i+3})))| }.
 \end{gather}
\end{itemize}

\begin{figure}[t]\centering
\includegraphics[height=6.5cm]{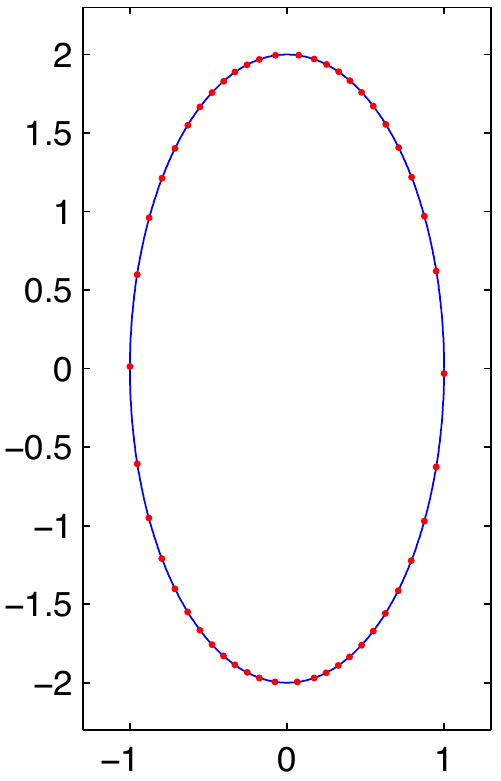}
\caption{The ellipse used as a test case, showing 50 points equally spaced with respect to M\"obius arclength.}\label{fig:ellipse}
\end{figure}

\begin{figure}[t!]\centering
\includegraphics[width=9.2cm]{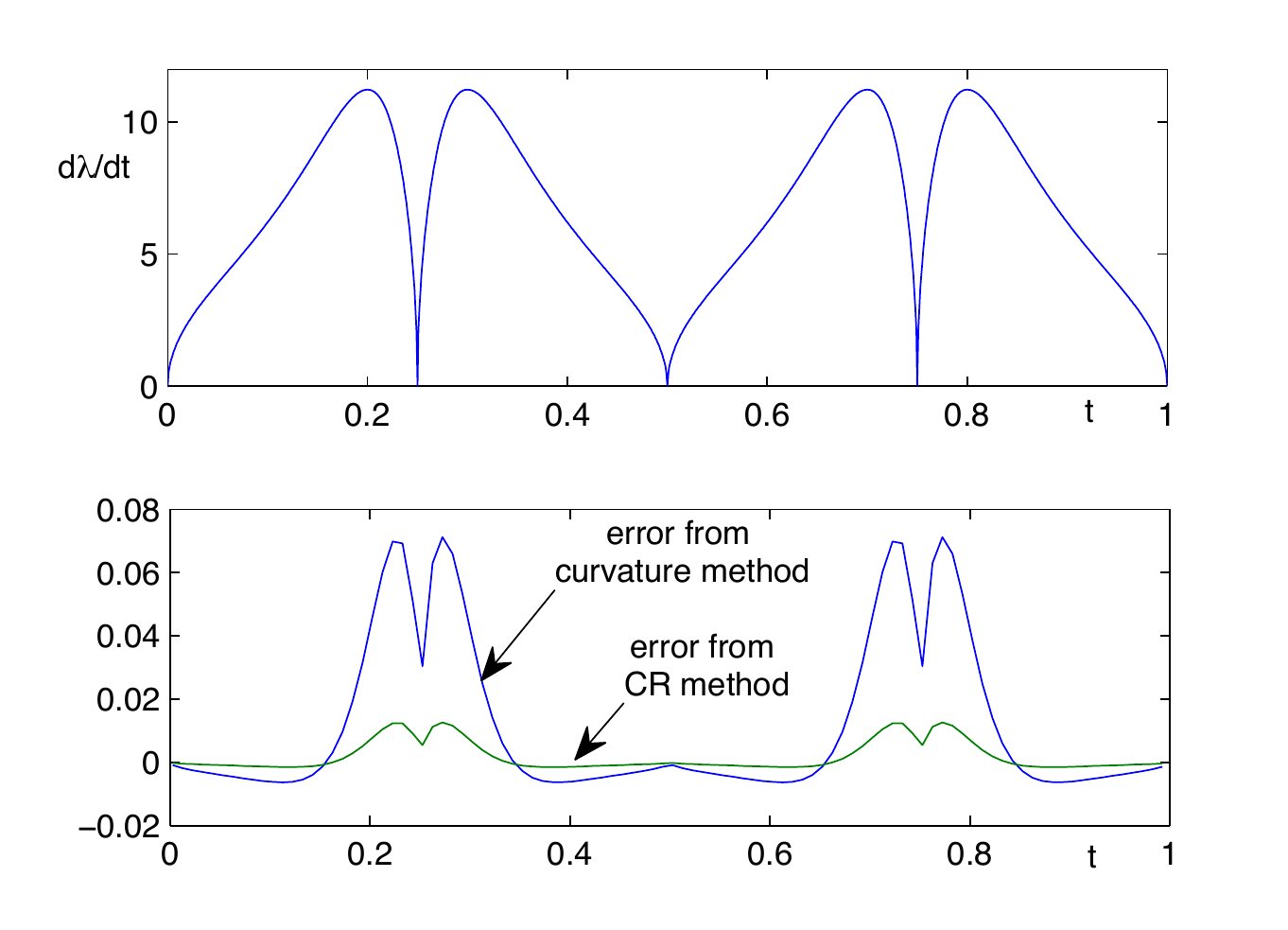}
\caption{The M\"obius arclength density, $\d\lambda/\d t$, for the ellipse, and its errors when calculated by the two approximations (\ref{eq:method1}) and (\ref{eq:method2}).}\label{fig:ell2}
\end{figure}

Away from vertices (points where $\d\lambda = 0$), both of these are second-order f\/inite dif\/ference approximations to the M\"obius length $\int_{t_{i+1}}^{t_{i+2}}\d\lambda$ of the arc $z([t_{i+1},t_{i+2}])$, or, dividing by $h$, to its arclength density $\d\lambda/\d t$. This can be established by expanding the approximations in Taylor series\footnote{A related approximation is $\d\lambda(t_{i+3/2}) \approx \sqrt{\frac{9}{2} | \mathrm{Im}(\CR(z_i,z_{i+1},z_{i+2},z_{i+3}))| }$. This is given in~\cite{barrett2007cauchy} but with an apparent error ($\frac{9}{2}$ replaced by~6).}. We test their accuracy as a function of the step size~$h$. The total M\"obius length of the curve is
\begin{gather*}
L := \int_0^1 d\lambda = \int_0^1 \frac{12\pi \sqrt{| \sin 4\pi t|} }{5+3 \cos 4\pi t}\, \d t = 6.86\quad (\mbox{to } 2 \mbox{ d.p.}).
\end{gather*}
The arclength density $\d\lambda/\d t$ is shown in Fig.~\ref{fig:ell2}, showing its 4 zeros at the vertices of the ellipse, where the ellipse is approximately circular, and its square-root singularities at the vertices. The M\"obius length is approximated by the trapezoidal rule, i.e., by $L_h := \sum\limits_{i=1}^n \d\lambda(t_{i+1/2})$. The error $L-L_h$ is expected to have dominant contributions of order $h^{3/2}$, due to the singularities at the vertices, and of order $h^2$, due to the f\/inite dif\/ferences used to approximate~$\d\lambda$.

In this example, the cross-ratio approximation has errors approximately 0.176 times those of the curvature approximation (see Fig.~\ref{fig:ell2}). However, due to some cancellations in this particular example, the curvature approximation actually gives a slightly more accurate approximation to the M\"obius length (see Fig.~\ref{fig:ellipseerr}). The dominant sources of error can be eliminated by two steps of Richardson extrapolation, f\/irst to remove the $\mathcal{O}(h^{3/2})$ error, and then to remove the $\mathcal{O}(h^2)$ error. This is highly successful and allows the calculation of the M\"obius length with an error of less than $10^{-10}$, even though it is singular and involves a 3rd derivative.

Next, we subject the ellipse to a variety of M\"obius transformations. The resulting shapes and errors are shown in Fig.~\ref{fig:mobiusdemo}. The errors increase markedly for the curvature method, which is not M\"obius invariant, but are unchanged for the cross-ratio method, which is M\"obius invariant. Thus, this experiment supports the argument of~\cite{calabi1998differential} that numerical approximations of invariants should themselves be invariant.

\begin{figure}[t]\centering
\includegraphics[scale=0.86]{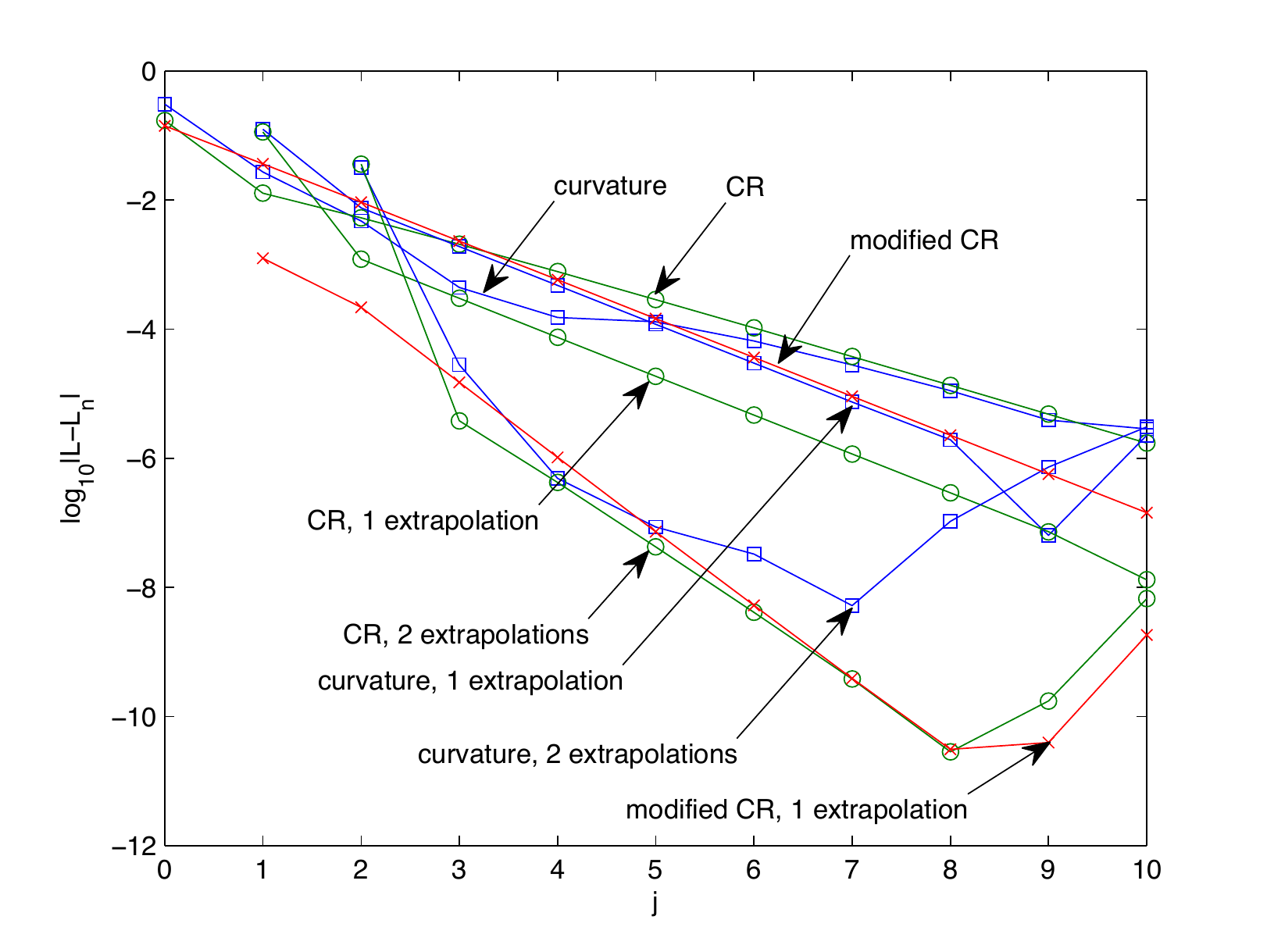}
\caption{The error in the two approximations of the total M\"obius length of the ellipse. Here there are $n := 25\cdot 2^j$ points on the ellipse.}\label{fig:ellipseerr}
\end{figure}

\begin{figure}[t!]\centering
\includegraphics[scale=1]{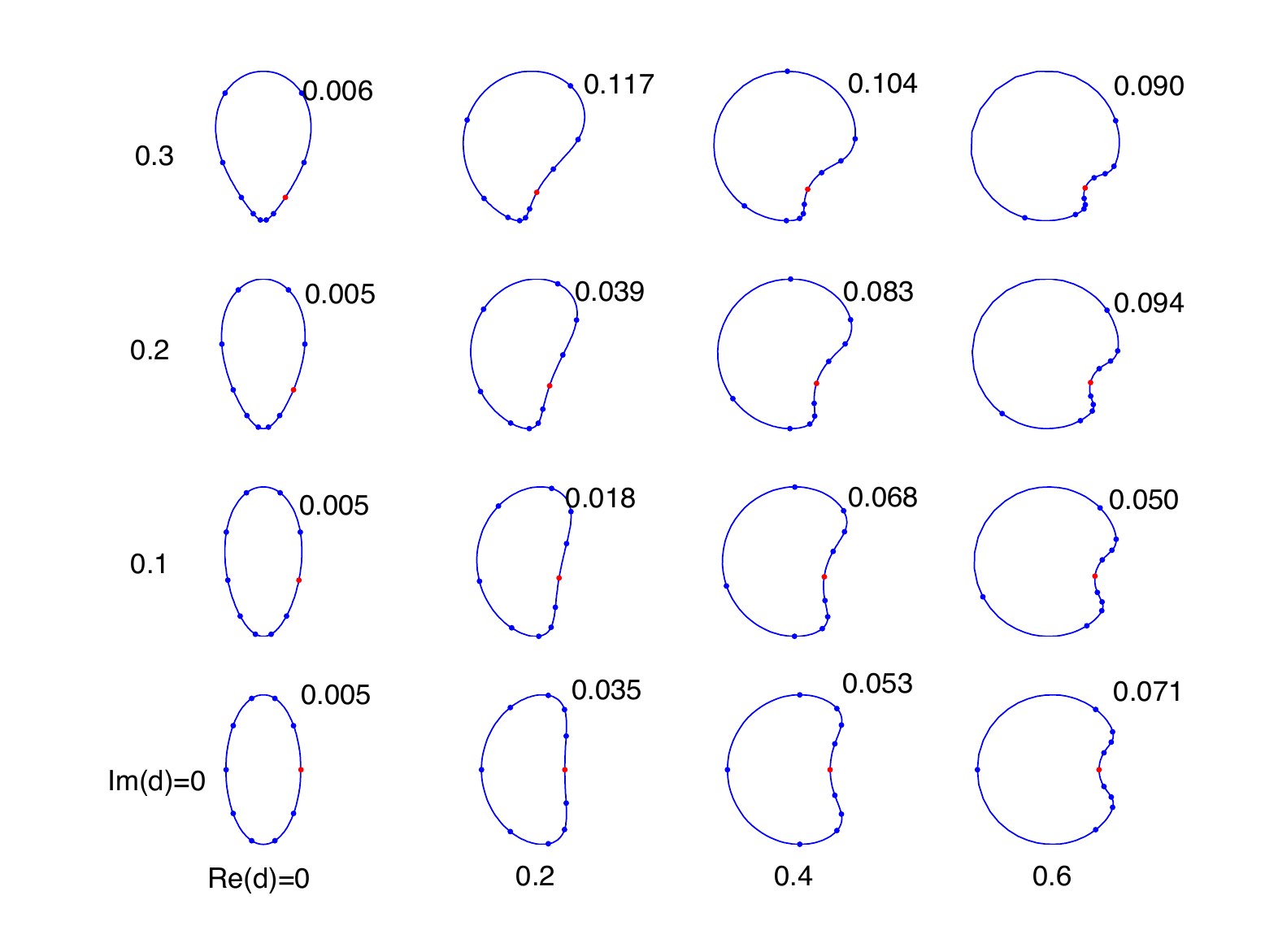}
\caption{The ellipse subjected to 16 dif\/ferent M\"obius transformations $z \mapsto \frac{z}{1 + d z}$. 10 landmarks, equally spaced in $\lambda$, are shown. Next to each shape is given the error in its M\"obius length as calculated by the curvature method. The error in the cross-ratio method is $0.005$ for all shapes, because it is M\"obius invariant. (The other M\"obius transformations are the Euclidean similarities, which are easy to visualize. Figures not to scale.)}\label{fig:mobiusdemo}
\end{figure}

\subsection[The M\"obius invariant ${\rm CR}(\lambda;z,\delta)$]{The M\"obius invariant $\boldsymbol{\CR(\lambda;z,\delta)}$}\label{sec:cr}
The method proposed in this paper for closed curves is to parameterize the curve by M\"obius arclength, giving $z(\lambda)$, and to use as an invariant the cross-ratio of all sets of 4 points a distance~$\delta$ apart. We call this the Shape Cross-Ratio, or SCR.

\begin{Definition} Let $z\colon S^1\to\mathbb{C}$ be a smooth curve. Let $L$ be its M\"obius length, and let $\tilde z \colon \R\to \mathbb{C}$ be an $L$-periodic function representing the curve parameterized by M\"obius arclength. The shape cross-ratio of $z$ is the $L$-periodic function $\SCR\colon \R\to\C$ def\/ined by{\samepage
\begin{gather*}
\SCR(\lambda;z,\delta) = \CR(\tilde z(\lambda), \tilde z(\lambda+\delta), \tilde z(\lambda+2\delta), \tilde z(\lambda + 3\delta)).
\end{gather*}
The shape cross-ratio {\em signature} of $z$ is the shape $\SCR(\R;z,\delta)\subset\C$.}
\end{Definition}

The shape cross-ratio is invariant under the M\"obius group, and sense-preserving reparameterizations of the curve act as translations in $\lambda$. (Reversals of the curve will be considered in Section~\ref{sec:reversals}.) The shape cross-ratio signature is invariant under the M\"obius group and under reparameterizations of the curve.

For an $L$-periodic function $f\colon \R\to\mathbb{C}$ we will denote its Fourier coef\/f\/icients by
\begin{gather*}
\F(f)_n = \frac{1}{L} \int_0^L f(t) e^{-2 \pi {\rm i} n t / L} \d t.
\end{gather*}
The translation $t\mapsto t+c$ acts on the Fourier coef\/f\/icients as $\F(f)_n \mapsto e^{-2\pi {\rm i} c n/L} \F(f)_n$. The Fourier amplitudes $|\F(f)_n|^2$ are invariant under translations, and can be used to recognize functions up to translations, but are clearly not a complete invariant: for a function discretized at $N$ equally-spaced points, and using the DFT, the space of orbits has dimension $2N-1$ and we have only $N$ invariants. The bispectrum~\cite{kakarala2012bispectrum} $\F(f)_m\F(f)_n\F(f)_{-m-n}$ is better, being complete on functions all of whose Fourier coef\/f\/icients are non-zero, but it is a very large set of invariants. Other invariants are $\F(\phi_1\circ f)_n \F(\phi_2\circ f)_{-n}$ for any functions $\phi_{1,2}$. Each such choice provides $2N$ invariants. The choice of $\phi_1$ and $\phi_2$ determines which aspects of $f$ are measured by the invariant. If necessary, several such pairs may be used.
\begin{Definition}
The Fourier cross-ratio of the shape $z$ is
\begin{gather}\label{eq:fcr1}
\FCR(\cdot;z,\delta)\colon \ {\mathbb Z}\to\mathbb{C},\qquad\!\! \FCR(n;z,\delta) = \F(\phi_1\circ\SCR(\cdot;z,\delta))_n \F(\phi_2\circ\SCR(\cdot;z,\delta))_{-n},\!\!
\end{gather}
where the Fourier transforms are based on the M\"obius length $L$ of $z$.
\end{Definition}

In the numerical illustrations we use
\begin{gather}\label{eq:fcr2}
 \phi_1(w) = \frac{w}{\sqrt{1+|w|^2}},\qquad \phi_2(w) = \phi_1(w)^2
 \end{gather}
 and the distance between the invariants of two shapes $z_1$ and $z_2$ given by
 \begin{gather} \label{eq:fcr3}
 \| \FCR(\cdot;z_1,\delta)-\FCR(\cdot;z_2,\delta) \|_2.
 \end{gather}
The motivation here is that the cross-ratio becomes arbitrarily large when two dif\/ferent parts of the curve approach one another. If left untouched (i.e., if we use just $\phi_1(w)=w$), then these large spikes in $\SCR(\lambda;z,\delta)$ will dominate all other contributions to the shape measurement. By scaling them using~\eqref{eq:fcr2}, they will still contribute to the description of the shape, but in a way that is balanced with respect to other parts of the shape. $\phi_1$ and $\phi_2$ take values in the unit disk and $\phi_2$ is sensitive to the main range of features of shapes; only values of $\SCR$ near~0 are suppressed, and these are rare.

The invariant $\SCR(\lambda;z,\delta)$ is smooth on simple closed curves, and also on most curves with self-intersections (blow-up requires the close approach of two points M\"obius distance $n\delta$ apart). It is locally complete, as given $z([a,a+3\delta])$, the invariant $\SCR(\lambda)$ determines $z$. We do not know if it is globally complete, i.e., if, given $\SCR(\lambda)$ which is the invariant of some shape, the shape can be determined up to M\"obius transformations, because this requires solving a nonlinear functional boundary value problem. Subject to this restriction, the invariant $\FCR(n;z,\delta)$ is complete except on a residual set of shapes (those for which enough Fourier coef\/f\/icients in~\eqref{eq:fcr1} are zero).

We will f\/irst study the numerical approximation of $\SCR$ and then study its use in recognizing shapes modulo M\"obius transformations.

The numerical experiments in Section~\ref{sec:ellipse} convinced us to approximate the M\"obius arclength using the cross-ratio. However, when combined with piecewise linear interpolation to locate points on the curve the required distance $\delta$ apart, we found that the resulting values of $\SCR(\lambda;z,\delta)$ did not converge as $h\to 0$. This is due to accumulation of errors along the curve, which arise particularly at the vertices due to the singularities there. This prompted us to develop a more ref\/ined interpolation method that takes into account the singularities of $\d\lambda/\d t$ at the vertices. We call it the {\em modified cross-ratio method}:

\begin{enumerate}\itemsep=0pt
\item Calculate the square of the M\"obius arclength density at the centre of each cell, as
\begin{gather*} \left(\frac{\d\lambda}{\d t}\right)^2_{i+3/2} = 6 \operatorname{Im}(\log(\CR(z_i,z_{i+1},z_{i+2},z_{i+3}))). \end{gather*}
If the curve is smooth, this is a smooth function.
\item Let $f(t)$, $0\le t<1$, be the piecewise linear interpolant of $\left(\frac{\d\lambda}{\d t}\right)^2_{i+3/2}$.
\item Calculate $\lambda(t)=\int_0^t \sqrt{|f(\tau)|}\, \d\tau$ and its inverse, $t(\lambda)$, used in locating the parameter values at which points a~desired length apart are located, {\em exactly} (we omit the formulas).
\item The desired points $z(t(\lambda+n \delta))$ are calculated using linear interpolation from the known values $z(i h)$.
\item The cross-ratio is evaluated at $N$ points equally spaced in $\lambda$, giving $\SCR(i L /N;z,\delta)$ for $i=1,\dots,N$, and the Fourier invariant $\FCR$ evaluated using two FFTs.
\end{enumerate}
The resulting cross-ratio is globally second-order accurate in $h$. Its accuracy could be increased for smooth curves using higher order interpolation, but the calculation of the inverse $t(\lambda)$ would be much more complicated and the method would be less robust.

The error in the length $L$ of the ellipse used in Section \ref{sec:ellipse} as calculated by this method is shown in Fig.~\ref{fig:ellipseerr}. It behaves extremely reliably over a wide range of scales of $h$; its error after one Richardson extrapolation is observed to be $\mathcal{O}(h^4)$, which indicates that the singularities at the vertices have been completely removed.

The parameter $\delta$ is the length scale on which $\SCR(\lambda;z,\delta)$ describes the shape. However, if $\delta\to 0$ then $\SCR(\lambda;z,\delta)\to \kappa_\Mob\pm \mathrm{i}$: the real part becomes extremely nonrobust and the imaginary part yields no information~\cite{Patterson1928}. In the experiments in this paper we have used $\delta = L/8$. Choosing $L/\delta\in\mathbb{Z}$ seems to yield somewhat improved accuracy, as the same values of~$z$ are used repeatedly.

\begin{figure}[t!]\centering
\includegraphics[width=125mm]{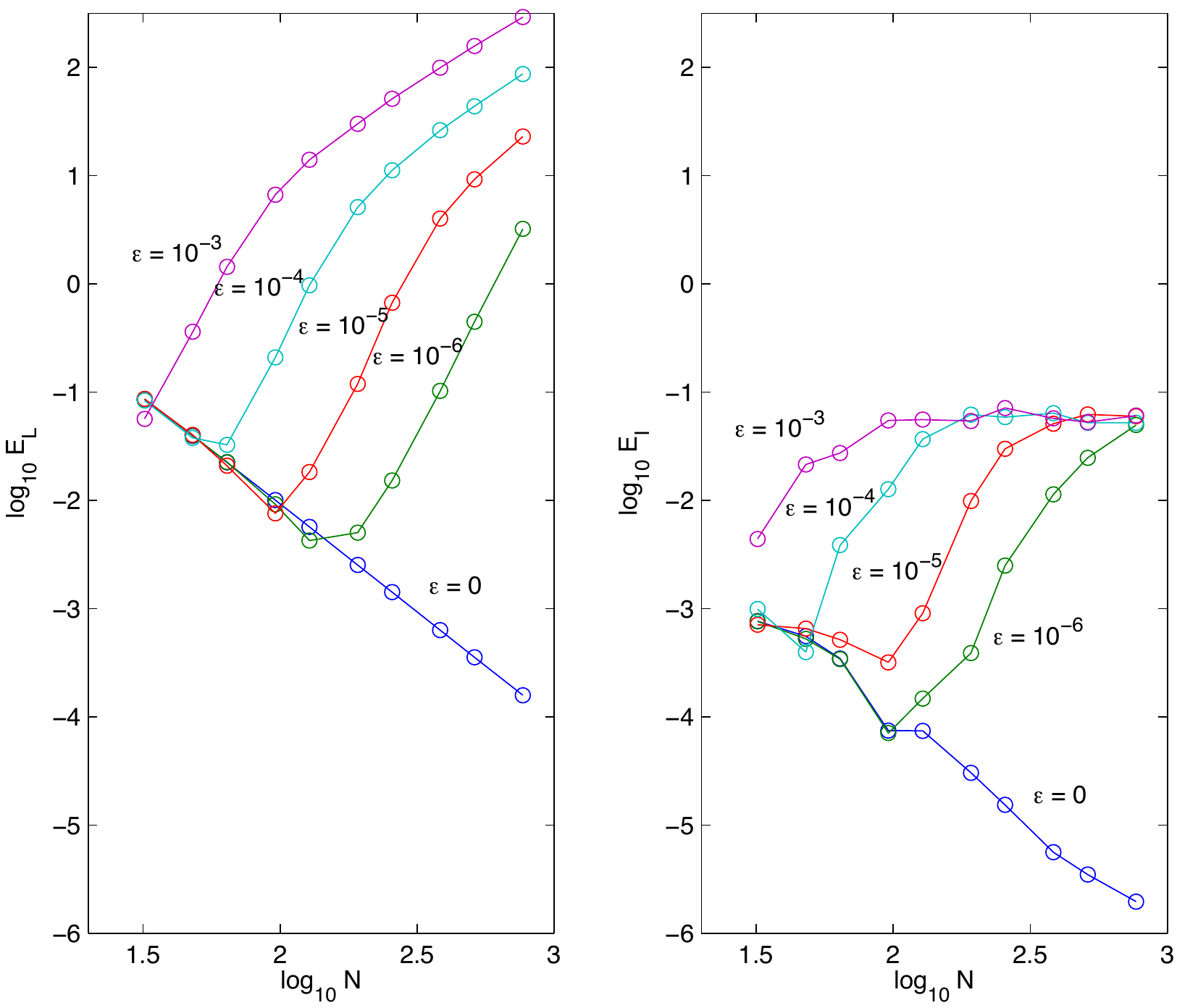}
\caption{Average error in $L$ (left) and $\FCR$ (right) for the ellipse shown in Fig.~\ref{fig:ellipse} as a function of the number of points $N$ and noise level $\varepsilon$.}\label{fig:ellipseerr2}
\end{figure}

As the method relies on parameterization by M\"obius arclength, it is unavoidably sensitive to noise in the data. We test its sensitivity for a noise model in which each point on the discrete curve is subject to normally distributed noise of standard deviation $\varepsilon$. The dependence of the error on $\epsilon$ and $N$ is shown for length and cross-ratio invariants in Fig.~\ref{fig:ellipseerr2}. Clearly, both are sensitive to relatively small amounts of noise. However, some positive features can also be seen:
\begin{itemize}\itemsep=0pt
\item[(i)] The errors in $L$ and $\FCR(n;z,\delta)$ are both $\mathcal{O}(h^2)$ in the absence of noise.
\item[(ii)] The errors in $\FCR(n;z,\delta)$ are much smaller than those in $L$~-- their relative errors are about 8 times smaller. (For the ellipse, $L\approx6.86$ and $\|\SCR(\cdot;z,\delta)\|_2\approx 0.65$.)
\item[(iii)] The error in $\FCR(n;z,\delta)$ appears to saturate at about 13\% as $\varepsilon$ increases and as $N$ increases.
\end{itemize}
Point (iii) is particularly striking. It appears to hold because (a) $\delta$ is chosen to be proportional to $L$, and thus when the noise is high, the chosen points stay roughly in their correct places; and (b) noise in the chosen points is averaged out by the Fourier transform, which remains dominated by its f\/irst few terms. This ef\/fect is illustrated in Fig.~\ref{fig:ellipseerr3} in which a single noise realization is illustrated for value of $\varepsilon$ from 0 to $10^{-2}$. Even though $L$ is overestimated by a~factor of~100, the signature cross-ratio $\SCR(\lambda;z,\delta)$ is still recognizable.

\begin{figure}[t!]\centering
\includegraphics[width=125mm]{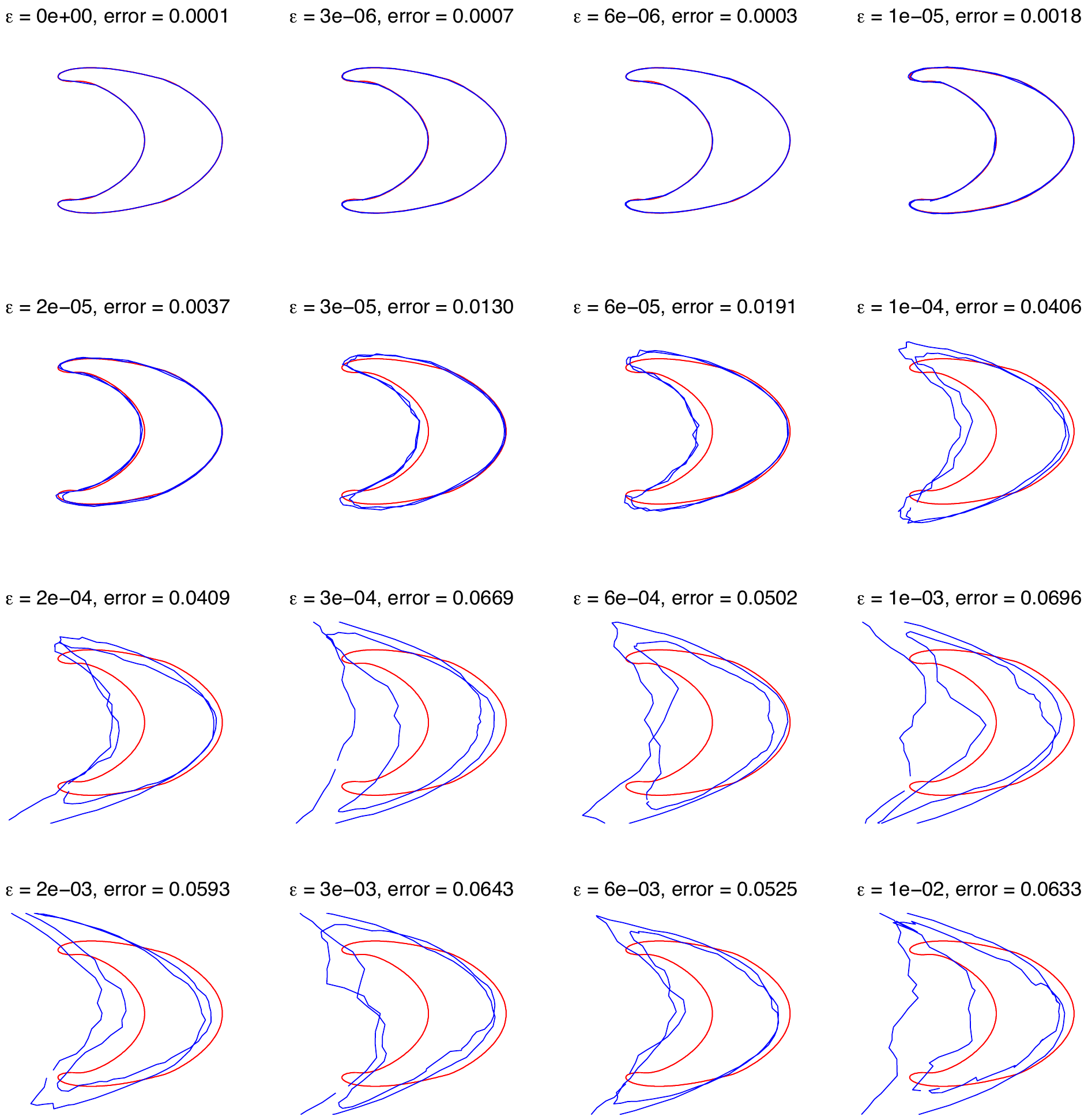}
\caption{The cross-ratio signature $\SCR\colon S^1\to\mathbb{C}$ (blue), and the error in its Fourier invariant $\FCR$, is shown for the ellipse discretized with $N=128$ points and various levels of noise. The exact signature is shown in red.}\label{fig:ellipseerr3}
\end{figure}

\begin{figure}[t!]\centering
\includegraphics[width=100mm]{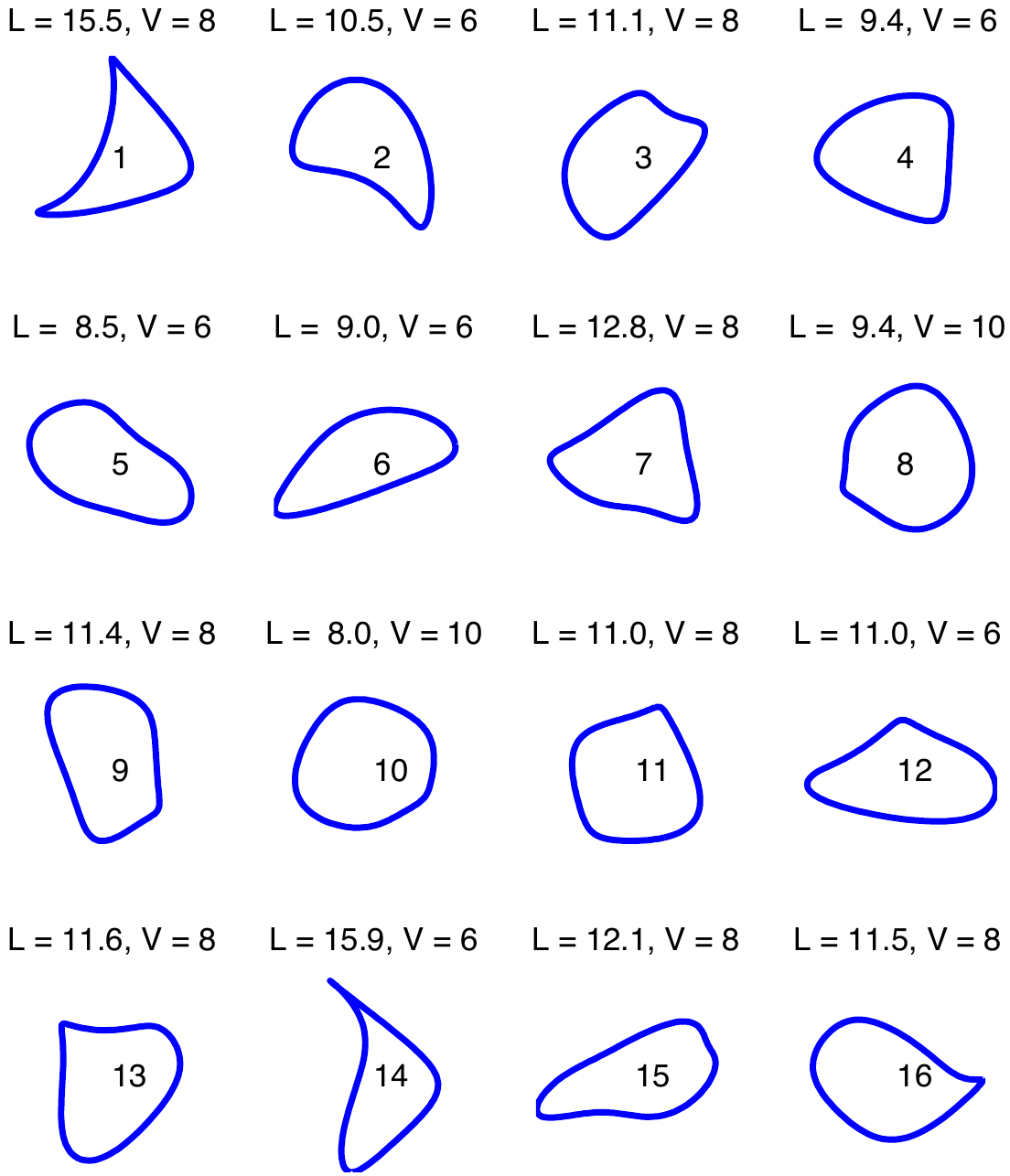}
\caption{The set of 16 shapes used in the numerical experiments. The shapes are shown to scale with the shape number marking the origin, along with the shape's M\"obius length $L$ and number of vertices $V$.}\label{fig:shapes}
\end{figure}

\begin{figure}[t!]\centering
\includegraphics[width=126mm]{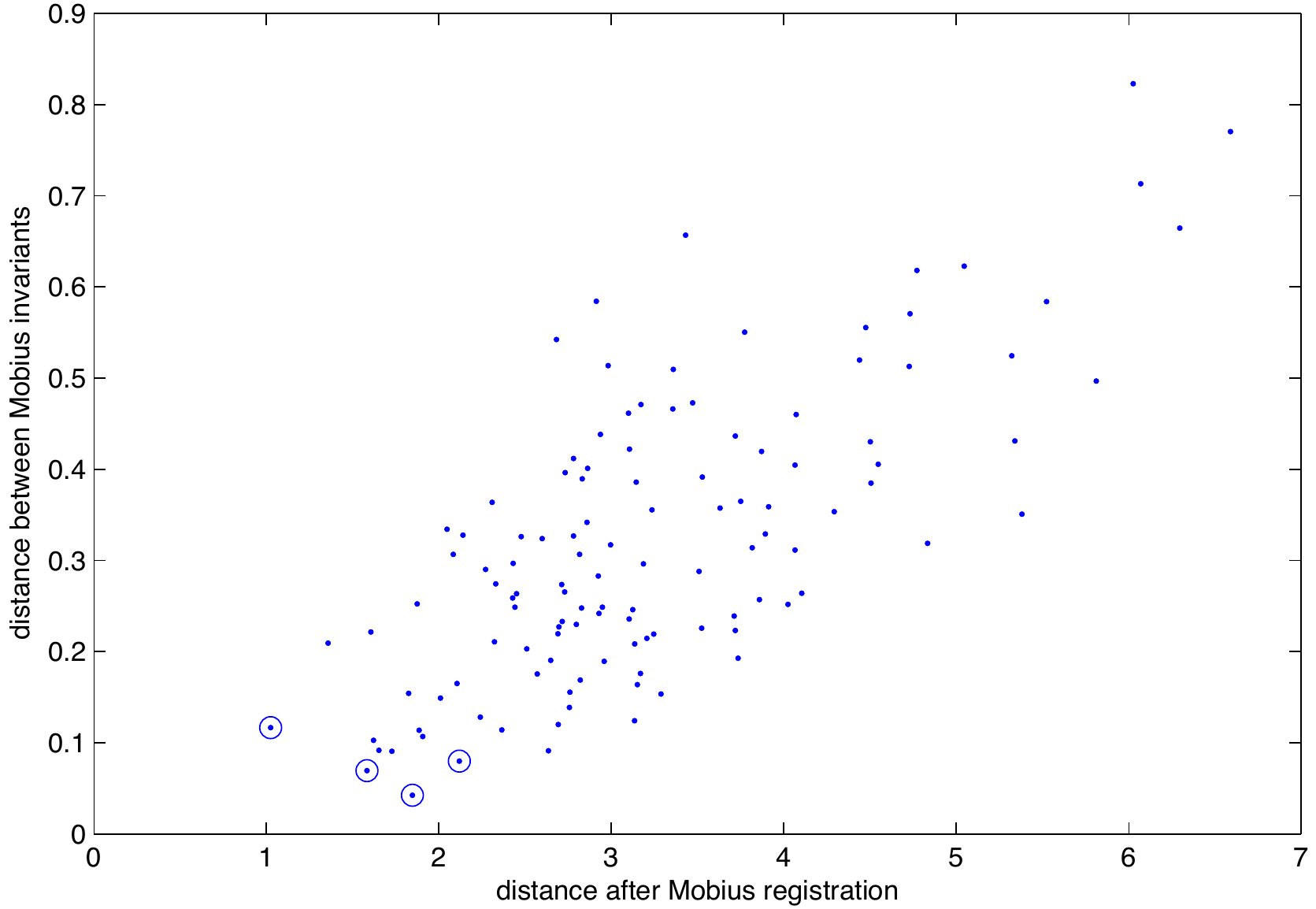}
\caption{Scatter plot of distance between all 120 pairs of 16 shapes with respect to (i) the $H^1$ distance between the shapes, as def\/ined in \eqref{eq:d1}--\eqref{eq:d3}, and (ii) the distance between their M\"obius invariants, def\/ined in \eqref{eq:fcr1}--\eqref{eq:fcr3}. The correlation coef\/f\/icient is 0.85. The registration between the 4 pairs marked with circles is illustrated in Figs.~\ref{fig:ref1213}--\ref{fig:ref0911}.}\label{fig:distancecorr}
\end{figure}

\subsection{Comparison with shape registration}\label{sec:comparison}
We now compare the results of the M\"obius invariant $\FCR(n;z,\delta)$ with direct registration of shapes. Given two shapes $z$ and $w$ we def\/ine the $G$-registration of $z$ onto $w$ as
\begin{gather}\label{eq:d1}
 r_G(z,w) = \min_{\varphi\in G\atop \psi\in\operatorname{Dif\/f}^+(S^1)} \| \varphi\circ z\circ\psi - w \| .
 \end{gather}
Dif\/ferent choices of norm in \eqref{eq:d1} will give dif\/ferent registrations; we have used the $H^1$ norm
\begin{gather}\label{eq:d2}
\|z\|_{H^1}^2 = \int_0^1 |z(t))|^2 + \alpha |z'(t))|^2\, \d t,
\end{gather}
where the constant $\alpha$ was chosen as $0.1$, a value which made both contributions to the norm roughly equal. One of the peculiarities of the M\"obius group is that~$z$ may register very well onto~$w$ while $w$ registers poorly onto~$z$. This happens when~$z$ has a distinguished feature which can be squashed, thus minimizing its contribution to~$r_G(z,w)$. Therefore in our experiments we use the `distance'
\begin{gather} \label{eq:d3}
d_G(z,w) = \max(r_G(z,w),r_G(w,z)).
\end{gather}

\begin{figure}[t!]\centering
\includegraphics[width=7.8cm]{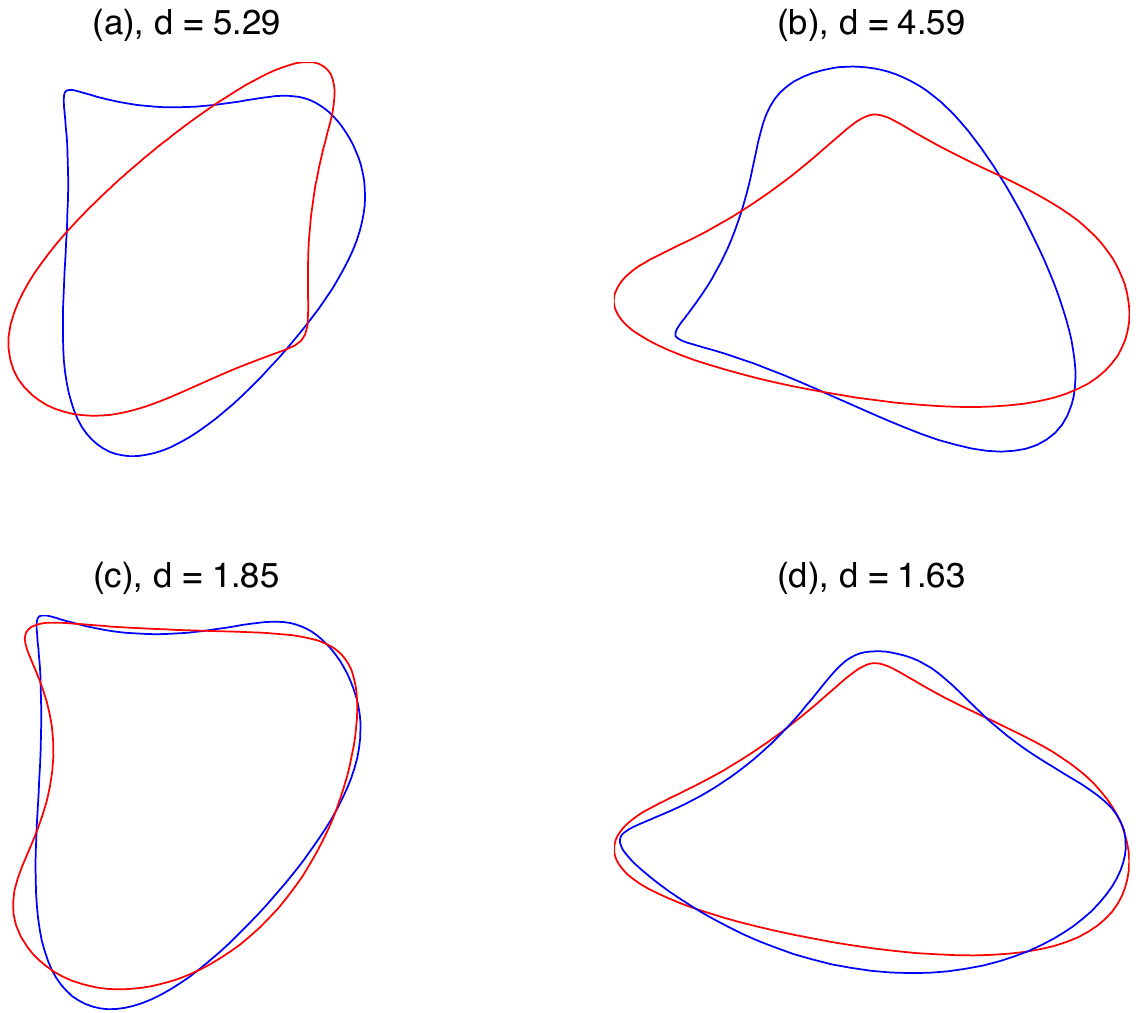}
\caption{Registration of shapes 12 (blue) \& 13 (red). This pair has the closest M\"obius invariants (distance 0.0426, see~\eqref{eq:fcr3}), is 9th closest after M\"obius registration, and 96th closest after similarity registration. (a)~Similarities act on blue shape; (b)~Similarities act on red shape; (c)~M\"obius acts on blue shape; (d)~M\"obius acts on red shape. Here $d=r_G(x,y)$ where $x$ and $y$ are the two shapes, see equation~\eqref{eq:d1}.}\label{fig:ref1213}\vspace{-2.5mm}
\end{figure}

\begin{figure}[t!]\centering
\includegraphics[width=7.9cm]{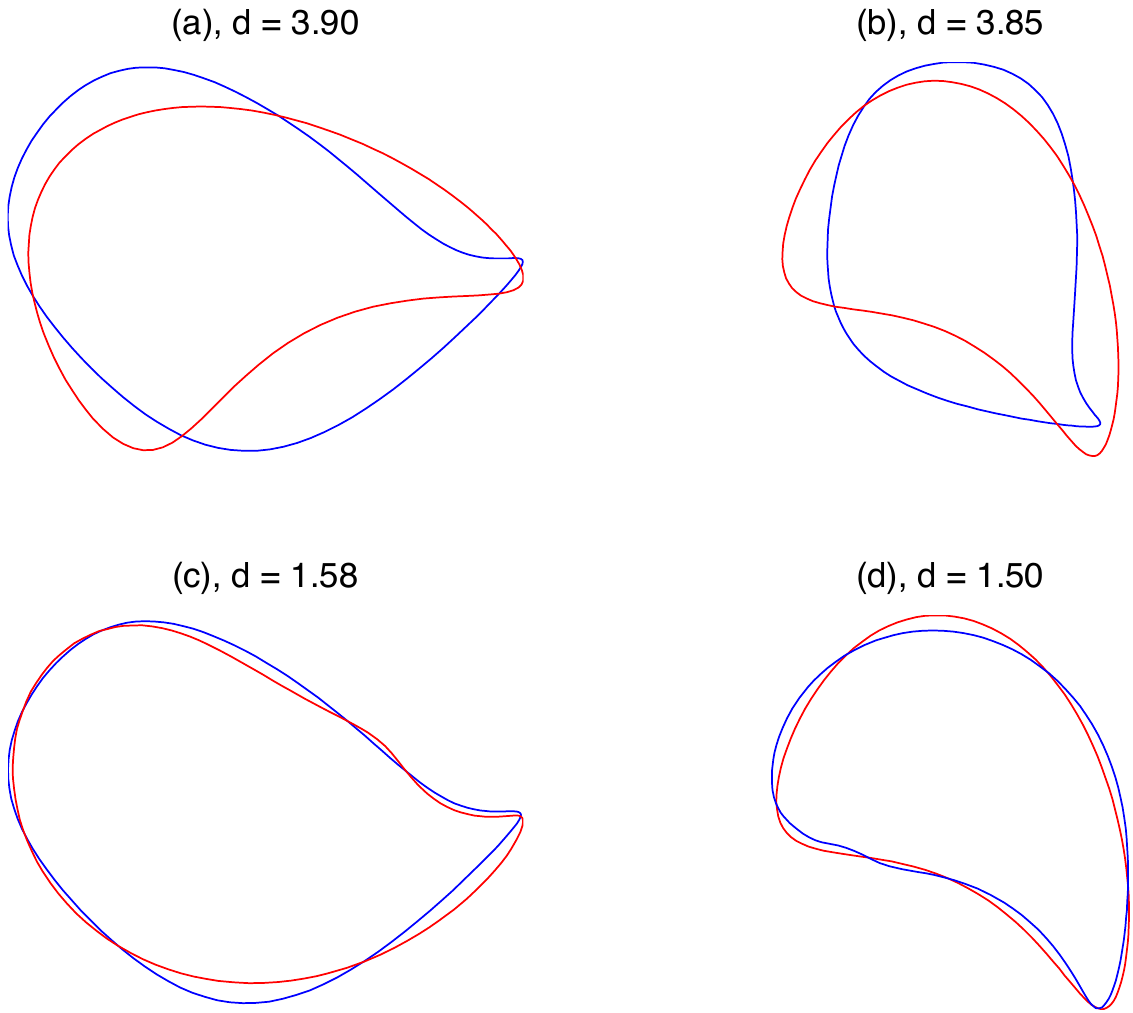}
\caption{Registration of shapes 2 (blue) \& 16 (red). This pair has the 2nd closest M\"obius invariants (distance 0.0697), see \eqref{eq:fcr3}), is 3rd closest after M\"obius registration, and 42rd closest after similarity registration. (a)~Similarities act on blue shape; (b)~Similarities act on red shape; (c)~M\"obius acts on blue shape; (d)~M\"obius acts on red shape. Here $d=r_G(x,y)$ where $x$ and $y$ are the two shapes, see equation~\eqref{eq:d1}.}\label{fig:ref0216}
\end{figure}

\begin{figure}[t]\centering
\includegraphics[width=7.9cm]{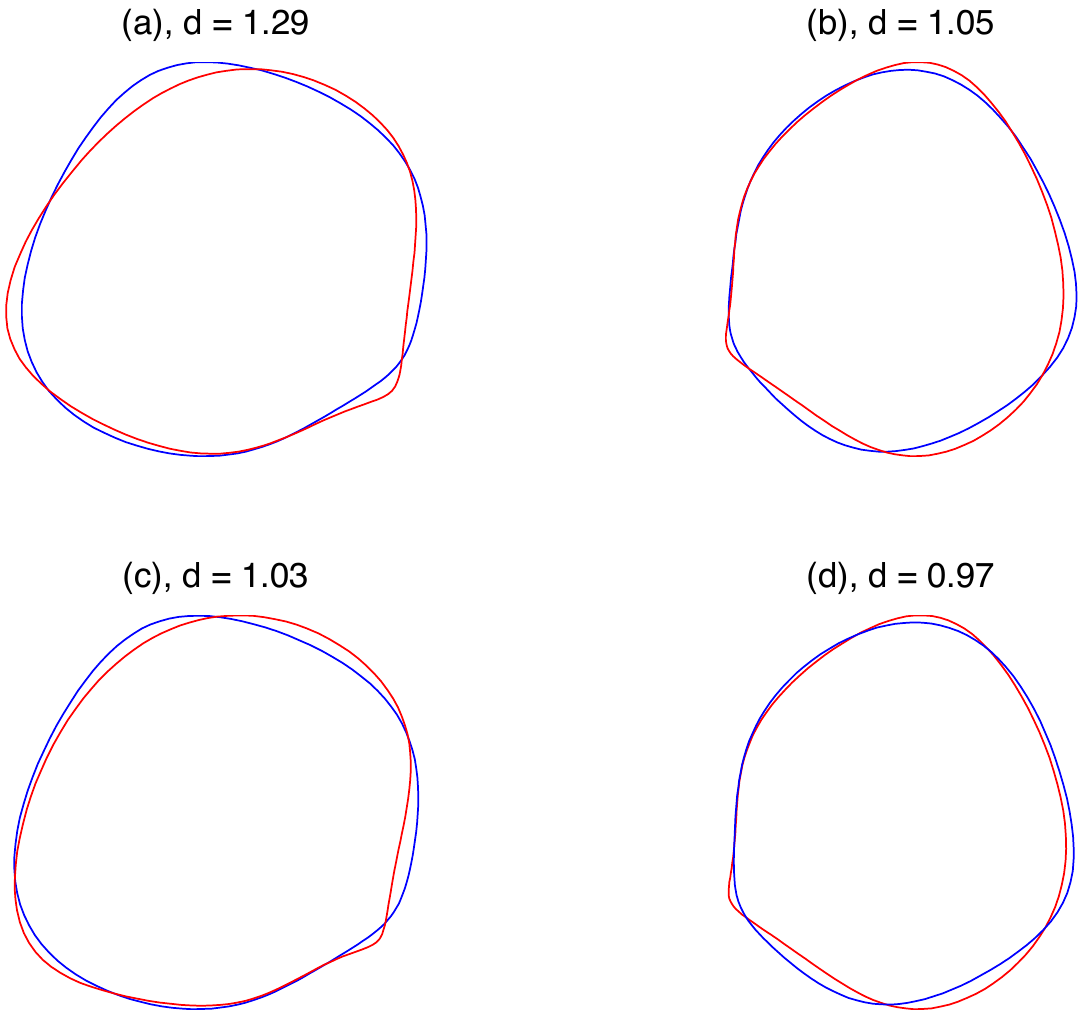}
\caption{Registration \looseness=1 of shapes 8 (blue) \& 10 (red). This pair is closest after M\"obius registration. It~has the 11st closest M\"obius invariants (distance 0.1165) and is the closest pair after similarity registration.
(a)~Similarities act on blue shape; (b)~Similarities act on red shape; (c)~M\"obius acts on blue shape; (d)~M\"obius acts on red shape. Here $d=r_G(x,y)$ where $x$ and $y$ are the two shapes, see equation~\eqref{eq:d1}.}\label{fig:ref0810}
\end{figure}

\begin{figure}[t]\centering
\includegraphics[width=7.9cm]{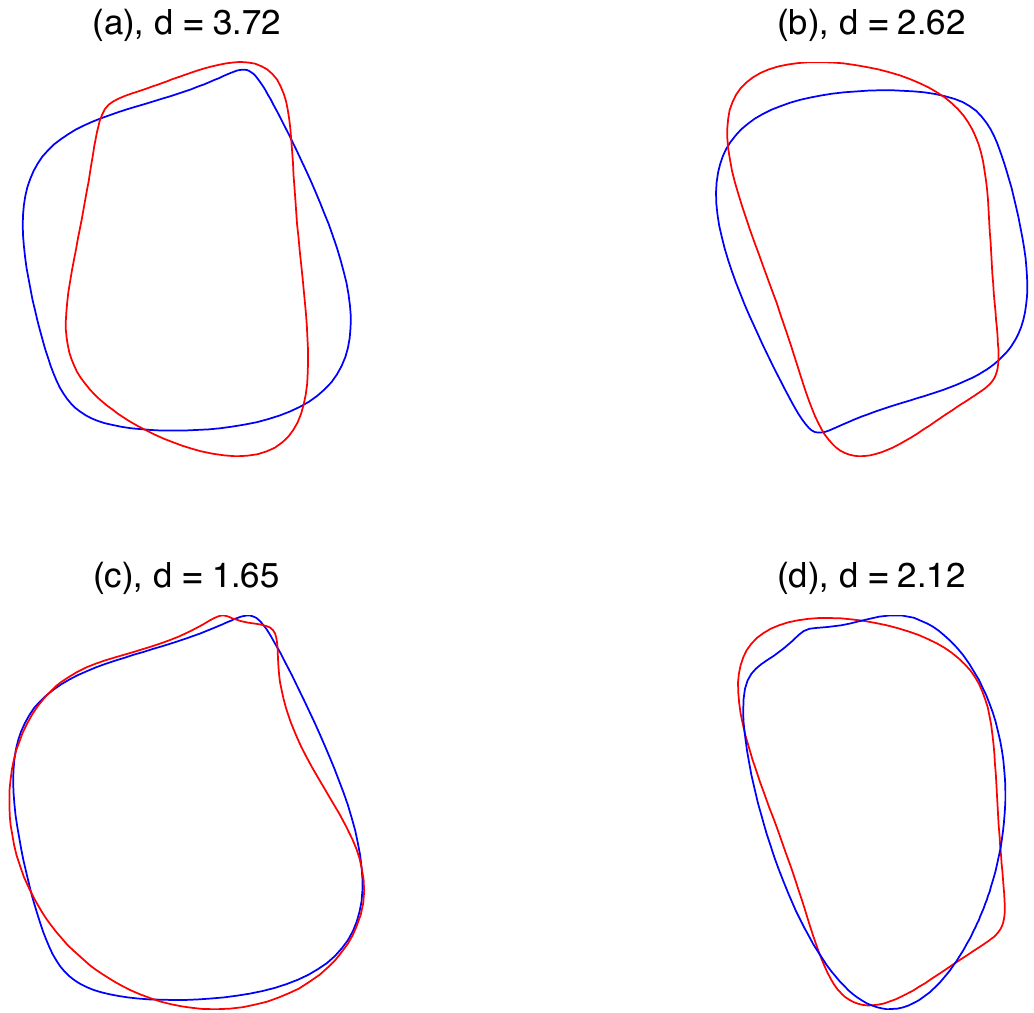}
\caption{Registration of shapes 9 (blue) \& 11 (red). This pair has the 3rd closest M\"obius invariants (distance 0.0798), is 9th closest after M\"obius registration, and 96th closest after similarity registration. (a)~Similarities act on blue shape; (b)~Similarities act on red shape; (c)~M\"obius acts on blue shape; (d)~M\"obius acts on red shape. Here $d=r_G(x,y)$ where~$x$ and~$y$ are the two shapes, see equation~\eqref{eq:d1}.}\label{fig:ref0911}
\end{figure}

Note that this should not be regarded as any kind of `ground truth' for the $G$-similarity of~$z$ and~$w$. It is not $G$-invariant. However, as we shall see, it does correspond remarkably well to the M\"obius invariants described earlier.

To calculate $d_G(z,w)$ numerically, we discretize $\operatorname{Dif\/f}^+(S^1)$ by piecewise linear increasing functions with 16 control points and perform the optimization using Matlab's {\tt lsqnonlin}, with initial guesses for $\varphi$ chosen to be each of 4 rotations and 3 scale factors for $\varphi$, and initial $\psi$ chosen to be the identity.

\looseness=-1 We generated 16 random shapes from a 14-dimensional distribution that favours smooth Jordan curves of similar sizes (where $U(0,2\pi)$ denotes uniform random numbers in the range~$0$ to~$2\pi$):
\begin{gather*} z(t) = \sum_{n=-4}^4 a_n \mathrm{e}^{2\pi \mathrm{i} n t},\qquad
\operatorname{Re}(a_n), \, \operatorname{Im}(a_n) \in N\big(0, 1/\big(1+|n|^3\big)\big), \qquad n\ne \pm 1,\\
\arg a_1\in U(0,2\pi),\qquad |a_1|=1,\qquad a_{-1}/a_1 \in U(0, 0.6).\end{gather*}
The shapes are shown in Fig.~\ref{fig:shapes}. They are all simple closed curves, which we take to be positively oriented. The 120 pairs of distinct shapes are registered in both directions, and the scatter plot between this distance and the 2-norm of the distance between their invariants is shown in Fig.~\ref{fig:distancecorr}. The correlation (0.85) is extremely striking, and suggests that this pair of measures may be related by bounded distortion \eqref{eq:boundeddistortion}, implying that they meet many of the requirements that we identif\/ied in Section~\ref{sec:req}; they are also quick to compute and provide a good numerical approximation.

Some examples of the registrations in the similarity and M\"obius groups are shown for four close pairs in Figs.~\ref{fig:ref1213}--\ref{fig:ref0911}. Including the invariants $L$ and $V$ in the list of invariants did not improve the correlation. Note that the errors in $\|\FCR(\cdot;z,\delta)\|$ observed in Fig.~\ref{fig:ellipseerr} are small enough to allow the separation of all but the closest pairs of shapes, regardless of the level of noise $\varepsilon$.

\subsection{Reversals and ref\/lections}\label{sec:reversals}

Orientation reversals and ref\/lections are examples of actions of discrete groups and can, in theory, be handled by any of the approaches in Section \ref{sec:invariants2}. First, consider the sense-reversing repara\-me\-te\-ri\-za\-tions, $z(t)\mapsto z(-t)$. These map the SCR invariant as $\SCR(\lambda;z,\delta) \mapsto \SCR(-\lambda;z,\delta)$ and hence $\FCR(n) \mapsto \FCR(-n)$. It is convenient to pass to the equivalent invariants:
\begin{gather*}
x_n = \begin{cases}
\FCR(n)-\FCR(-n), & n>0, \\
\FCR(n)+\FCR(-n), & n\le 0,
\end{cases}
\end{gather*}
for which $x_n$ is invariant for $n\le 0$ and $x_n\mapsto -x_{-n}$ for $n>0$. Suppose that we wish to identity curves with their reversals, that is, to work with unoriented shapes. Some options are the following:
\begin{enumerate}\itemsep=0pt
\item The moving frame method: the shapes are put into a reference orientation f\/irst. This is only possible if the problem domain is restricted suitably; in this case, for example, to simple closed curves, which can be taken to be positively oriented. This is the approach that we have taken in Sections~\ref{sec:cr} and~\ref{sec:comparison}. If the problem domain includes non-simple curves this approach may not be possible. For example, in the space of plane curves with the topology of a f\/igure 8, each such shape can be continuously deformed into its reversal; thus we cannot assign them orientations, since they vary continuously with the shape.
\item Finding a complete set of invariants: this is $x_n$ for $n\le0 $ and $x_i x_j$ for $i,j>0$. Again we see that the quotient by a relatively simple group action is expensive to describe completely using invariants.
\item Use an incomplete set of invariants that is ``good enough'': here, using $x_n$ for $n\le 0$ and $x_n x_{n+1}$ for $n>0$ is a possibility. This creates a complicated ef\/fect on the metric used to compare invariants.
\item As the group action is of a standard type, one can work in unreduced coordinates ($x_n$) together with a natural metric induced by the quotient, such as some function of $\|x\|-\|y\|$ and the Fubini--Study metric on projective space, $\cos^{-1}(|\bar x^T y|/\|x\|\|y\|)$.
\item Finally, and most easily in this case, for f\/inite groups one can represent points in the quotient as entire group orbits. A~suitable metric is then
\begin{gather*}
d(x,y) = \min_{g\in G}\|x - g\cdot y\|,
\end{gather*}
where $G$ is the group. Although this is impractical for large f\/inite groups, here $G$ is $\mathbb{Z}_2$.
\end{enumerate}
Similar considerations apply to recognizing shapes modulo the full inversive group, generated by the M\"obius group together with a ref\/lection, say $z\mapsto \bar z$. The ref\/lection maps $\FCR(n)\mapsto \overline{\FCR(n)}$ and hence is an action of the same type as reversal (a sign change in some components). The ref\/lection symmetry of the ellipse in Fig.~\ref{fig:ellipse}, for example, can be detected by the ref\/lection symmetry of the cross-ratio signature $\SCR(\lambda)$ in Fig.~\ref{fig:ellipseerr3}. (Its second discrete symmetry, a rotation by $\pi$, is manifested in Fig.~\ref{fig:ellipseerr3} by the signature curve retracting itself twice.)

The invariants developed here are for closed shapes. They can be adapted for other types of shapes (for example, shapes with the topology of two disjoint circles), but as the topology gets more complicated (for example, shapes consisting of many curve segments) the problem becomes signif\/icantly more dif\/f\/icult.

\section{M\"obius invariants of images}\label{sec:images}

Let $f \colon \C \to [0,1]$ be a smooth grey-scale image. Dif\/feomorphisms act on images by $\varphi\cdot f := f\circ \varphi^{-1}$. It is easy, in principle, to adapt the M\"obius shape invariant $\SCR$ to images by computing level sets of $f$, each of which is an invariant shape for which $\SCR$ can be calculated. In addition, if $\varphi$ is conformal, the orthogonal trajectories of the level sets, i.e., the shapes tangent to $\nabla f$, are also invariant shapes. In the neighbourhood of a simple closed level set, coordinates $(\lambda,\mu)$ can be introduced, where $\lambda$ is M\"obius arclength along the level set and $\mu$ is M\"obius arclength along the orthogonal trajectories. The quantity
\begin{gather}\label{eq:cri}
\CR(z(\lambda,\mu), z(\lambda,\mu+\delta), z(\lambda+\delta,\mu+\delta),z(\lambda+\delta,\mu)),
\end{gather}
calculated from the cross-ratio of 4 points in a square, is then invariant under the M\"obius group, and reparameterizations of the level set act as translations in $\lambda$.

In practice, however, the domain of this invariant is quite restricted. The topology of level sets is typically very complicated and the domain of~$f$ may be restricted, so that level sets can stop at the edge of the image. Restricting to level sets of grey-scales near the maximum and minimum of~$f$ helps, but this is a severe restriction. Instead, we shall show that the extra information provided by an image, as opposed to that provided by a~shape, determines a dif\/ferential invariant signature using only 3rd derivatives, compared to the 5th derivatives needed for dif\/ferential invariants of shapes. Because of this, we do not develop the cross-ratio invariant~(\ref{eq:cri}) any further here.

\begin{Proposition} Let $f \colon \C\to\R$ be a smooth grey-scale image. Let $R\subset\C$ be the regular points of $f$. Identify $x_1+{\rm i}x_2\in\mathbb{C}$ with $(x_1,x_2)\in\R^2$ so that $\nabla f$ is the standard Euclidean gradient. On $R$, define
\begin{gather*}
n := \frac{\nabla f}{ \| \nabla f\|},\qquad
\lambda_n := \frac{n\cdot \nabla (\nabla\times n)}{\|\nabla f\|^2},\qquad
\lambda_t := \frac{n\times \nabla(\nabla\cdot n)}{\|\nabla f\|^2}.\end{gather*}
Then
\begin{gather}\label{eq:sig}
(f, \lambda_n, \lambda_t)(R)
\end{gather}
is a subset of $\R^3$ that is invariant under the action of the M\"obius group on images.
\end{Proposition}
\begin{proof} As def\/ined above, $n$ is the unit vector f\/ield normal to the level sets of~$f$. Let $n^\perp$ be the unit vector tangent to the level sets given by $n^\perp_i = \varepsilon_{ij}n_j$. (Here $i,j=1,2$ and $\varepsilon_{ij}$ is the Levi-Civita symbol def\/ined by $\varepsilon_{11}=\varepsilon_{22}=0$, $\varepsilon_{12}=1$, $\varepsilon_{21}=-1$; we sum over repeated indices and write $n_{i,k} = \partial n_i/\partial x_j$.) From the Frenet--Serret relation $n_s = \kappa n^\perp$, where $s$ is arclength along the level sets, we have that the curvature of the level sets is
\begin{align*}
\kappa &= n^\perp\cdot n_s = n^\perp\cdot((n^\perp\cdot\nabla)n)= \varepsilon_{ij} n_j \varepsilon_{kl}n_l n_{i,k}
 =(\delta_{ik}\delta_{jl} - \delta_{il}\delta_{jk}) n_j n_l n_{i,k} \\
&= n_j n_j n_{k,k} - n_k n_i n_{i,k} = n_{k,k} \quad \hbox{\rm (because\ }n_j n_j = 1\Rightarrow n_i n_{i,k}=0\hbox{\rm\ for all\ }k)\\
&= \nabla \cdot n.
\end{align*}
Recall that the M\"obius arclength of the level sets of $f$ is $\mathrm{d}\lambda := \sqrt{|\kappa_s|}\mathrm{d}s$. Under any conformal map, the scaling along and normal to the level sets is the same, and thus $\mathrm{d}s$ and $1/\|\nabla f\|$ both scale by the same factor. Therefore $\sqrt{|\kappa_s|}/\|\nabla f\|$ is invariant, as is its square $|\kappa_s|/\|\nabla f\|^2$. The sign of $\kappa_s$ is also invariant under M\"obius transformations, resulting in the given invariant $\lambda_t = \kappa_s/\|\nabla f\|^2$.

The invariant $\lambda_n$ arises in the same way from the orthogonal trajectories, whose curvature is $\nabla\cdot n^\perp = \nabla\times n$.
\end{proof}

\begin{Example} As a test image we take the smooth function
\begin{gather}\label{eq:f}
f(x,y) = e^{-4x^2-8\big(y-0.2 x - 0.8 x^2\big)^2}
\end{gather}
and calculate its invariant signature before and after the M\"obius transformation with parameters
\begin{gather}\label{eq:abcd}
 a = 0.9 + 0.1\mathrm{i},\qquad b = 0.1,\qquad c = 0.1+0.4\mathrm{i},\qquad d = 1.
 \end{gather}
on the domain $[-1,1]^2$. The invariants are approximated by f\/inite dif\/ferences with mesh spacing $1/80$, corresponding to $161\times 161$ pixel images. The invariants are shown as functions of $(x,y)$ in Fig.~\ref{fig:lamn} for~$\lambda_n$ and Fig.~\ref{fig:lamt} for~$\lambda_t$. The resulting signature surfaces, shown for~$f$ in Fig.~\ref{fig:sig3D} in~$\R^3$, are quite complicated. A useful way to visualize and compare them is shown in Fig.~\ref{fig:sig}. For example, one can plot the contours of~$f$ in the $(\lambda_n,\lambda_t)$ plane, and similarly for other projections. This enables a sensitive comparison of the signatures of the image and its M\"{o}bius transformed version and reveals that they are extremely close.
\end{Example}

\begin{figure}[t!]\centering
\includegraphics[width=7.7cm]{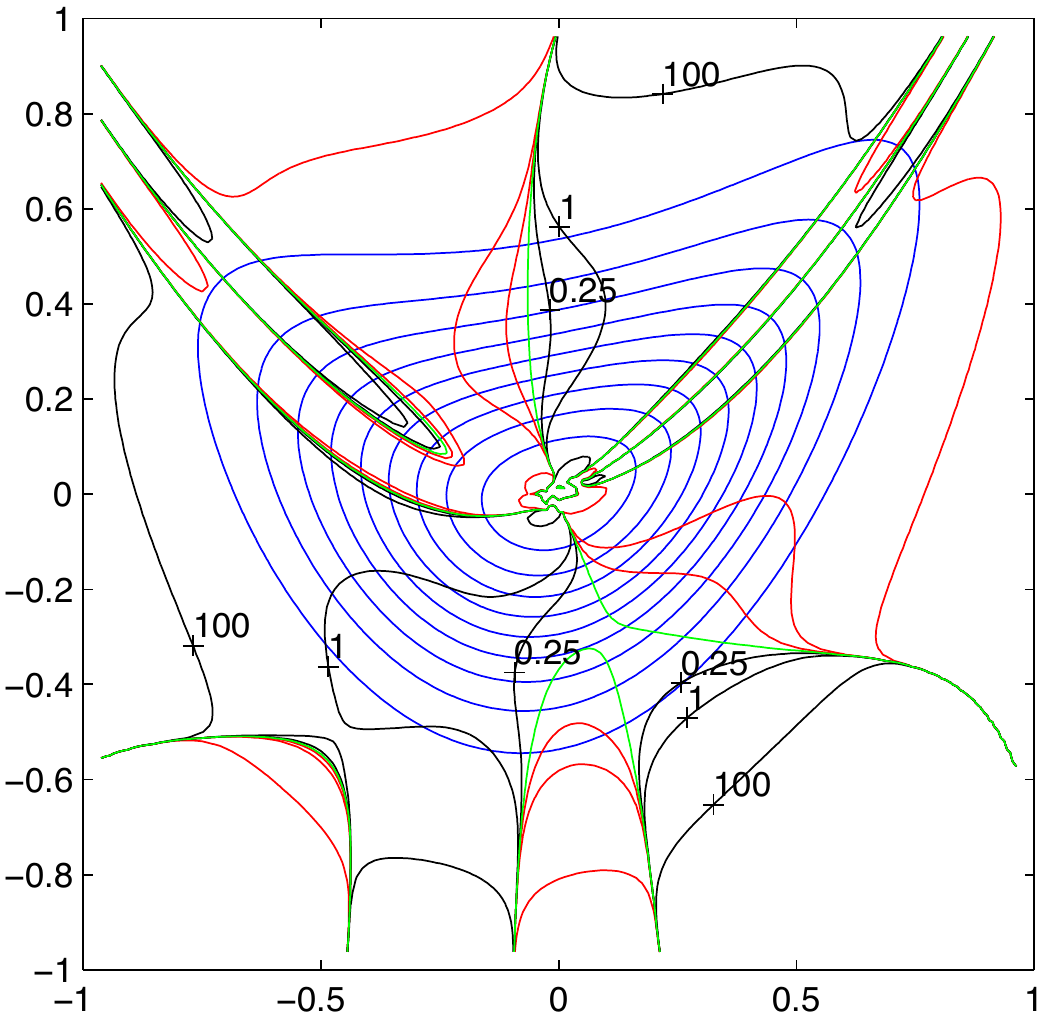}
\smallskip

\includegraphics[width=7.7cm]{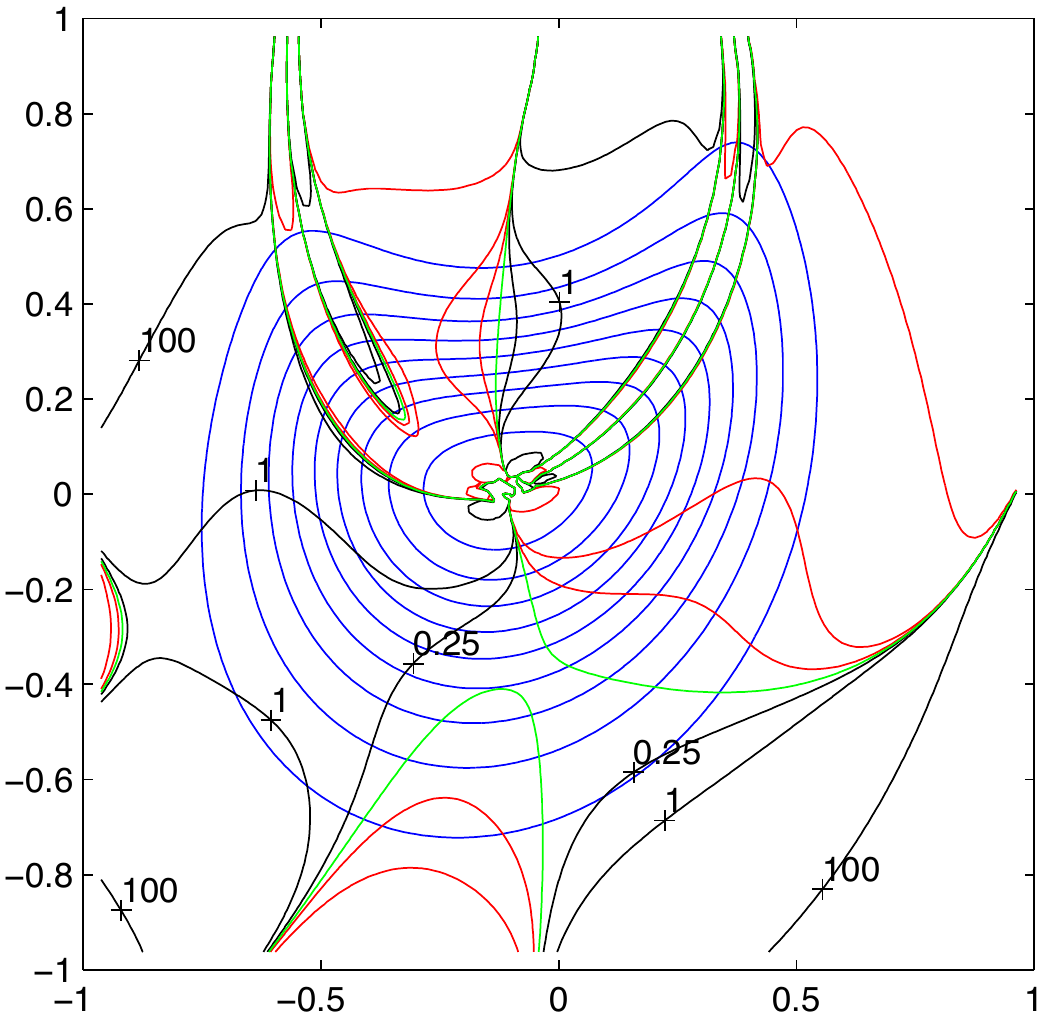}
\caption{Contours 0.1,0.2,\dots,0.9 of a function are shown in blue, together with its invariant $\lambda_n$: contour 0 in green, contours $-0.25$, $-1$, and $-100$ in red, and contours 0.25, 1, and 100 in black. Top: function $f$ from (\ref{eq:f}). Bottom: M\"obius related function $f\circ\varphi^{-1}$, parameters in (\ref{eq:abcd}). The invariance can be seen, along with the way that $\lambda_n$ typically blows up as $\nabla f\to 0$. A small discretization error is visible in the top f\/igure: the saddle point near $(-0.5,-0.5)$ has $\lambda_n\approx 1.07$, whereas the exact value is 0.94. This results in the wrong topology of the $+1$ contour (cf.\ bottom f\/igure near $(-0.8,-0.2)$).}\label{fig:lamn}
\end{figure}

\begin{figure}[t!]\centering
\includegraphics[width=7.8cm]{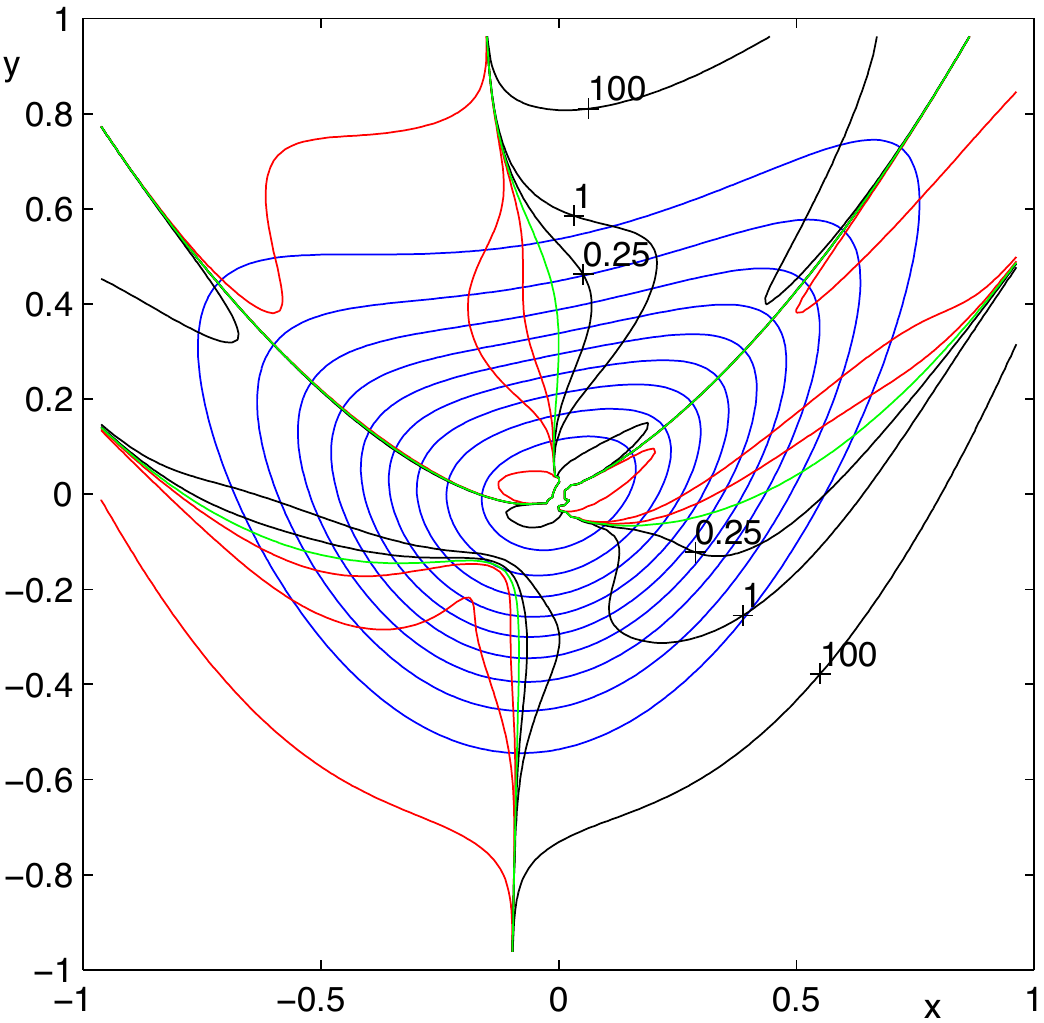}
\smallskip

\includegraphics[width=7.8cm]{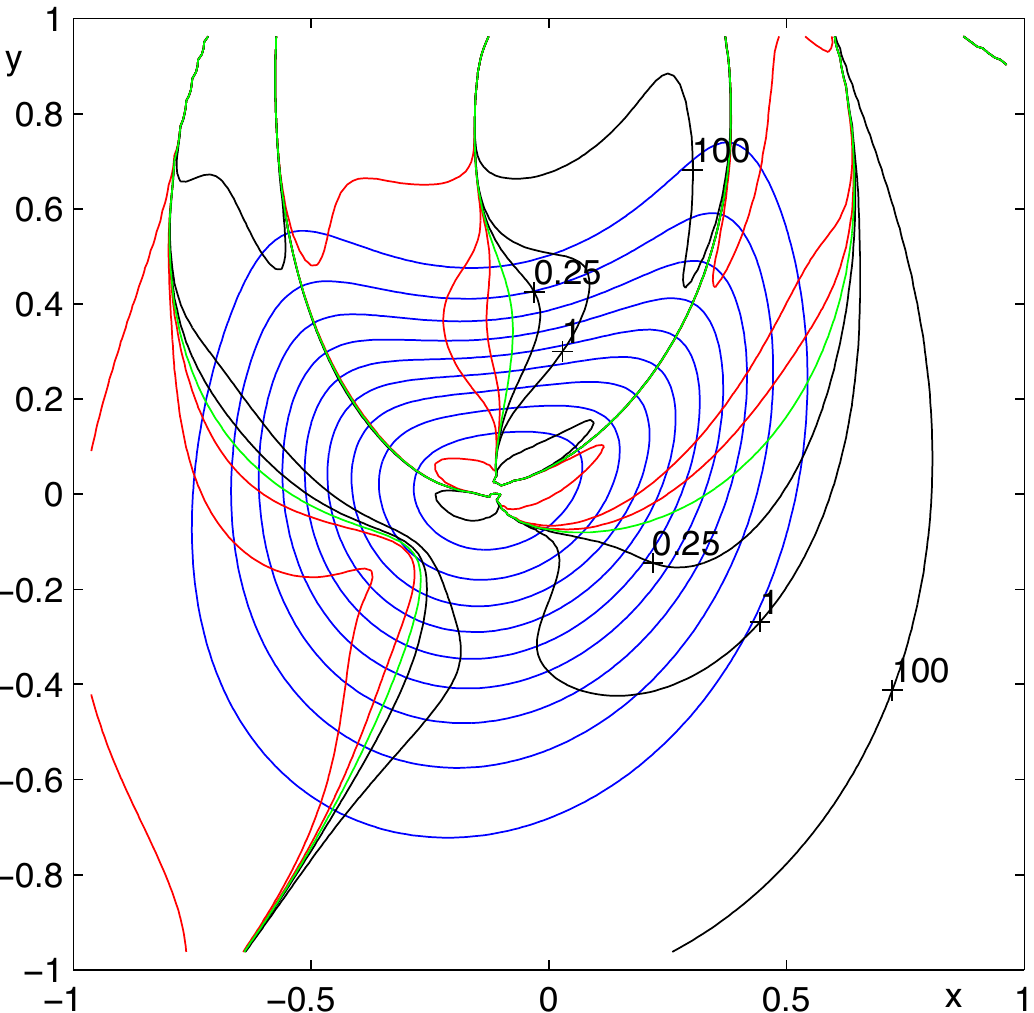}
\caption{Contours 0.1,0.2,\dots,0.9 of a function are shown in blue, together with its invariant $\lambda_t$: contour 0 (which locates vertices (points of stationary curvature) of the level sets) in green, contours $-0.25$, $-1$, and $-100$ in red, and contours 0.25, 1, and 100 in black. Top: function $f$ from (\ref{eq:f}). Bottom: M\"obius related function $f\circ\varphi^{-1}$, parameters in (\ref{eq:abcd}). }\label{fig:lamt}\vspace{-2.5mm}
\end{figure}

\begin{figure}[t!]\centering
\begin{minipage}[b]{0.4\textwidth}
\centering
\includegraphics[width=0.6\textwidth]{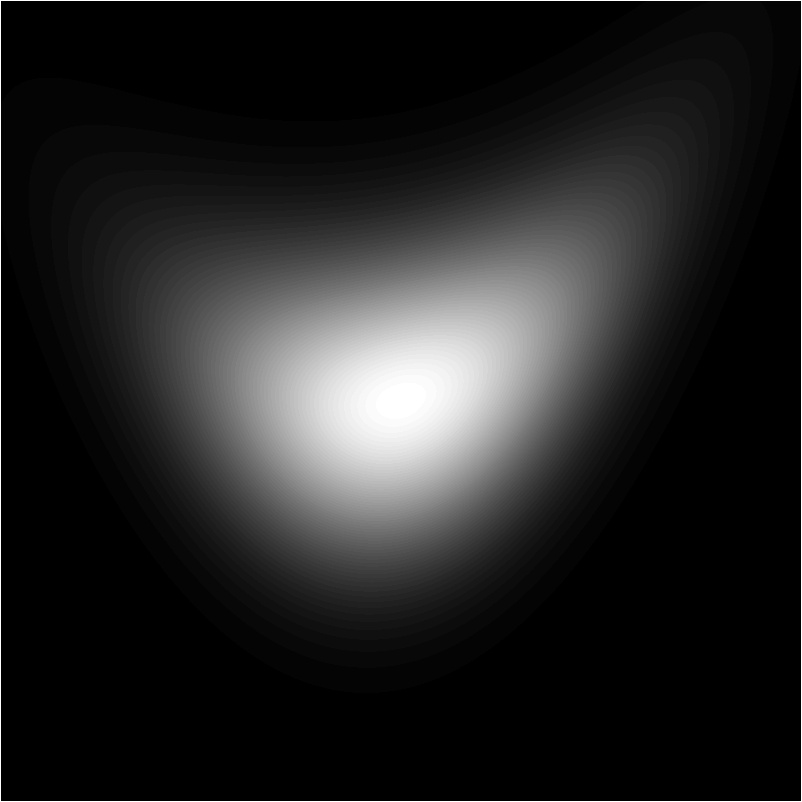}
\begin{overpic}[width=0.95\textwidth]{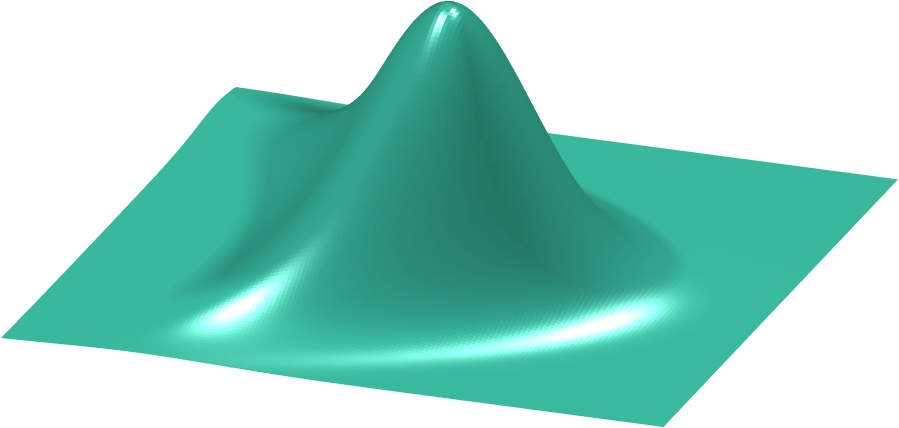}
\put (75,-5) {\small $-1$}
\put (90,11) {\small$x$}
\put (102,25) {\small$1$}
\put (0,5) {\small$1$}
\put (30,0) {\small$y$}
\end{overpic}
\end{minipage}
\begin{minipage}[b]{0.6\textwidth}\centering
\includegraphics[width=95mm]{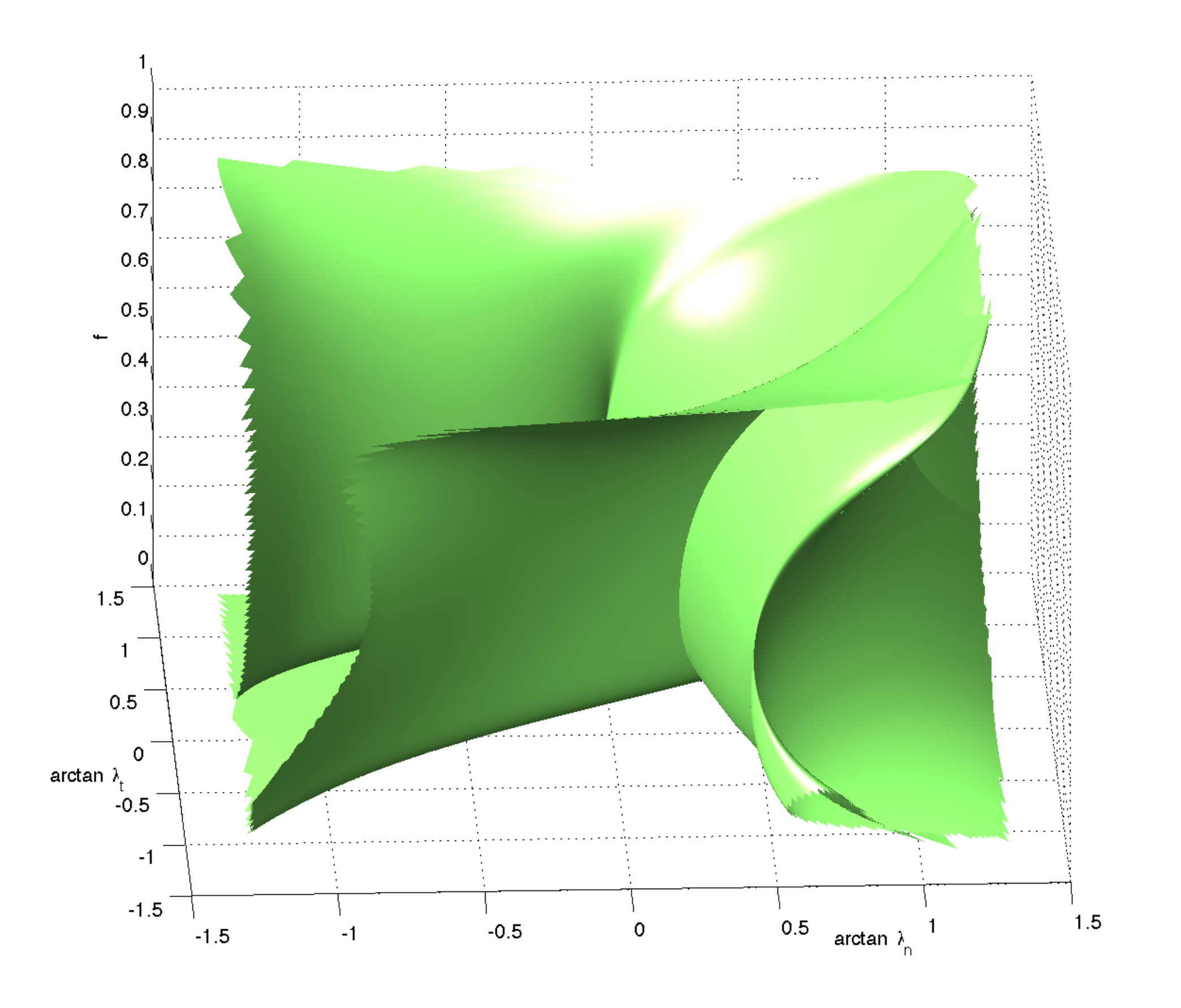}
\end{minipage}
\caption{The sample image def\/ined in equation~(\ref{eq:f}) is shown in grayscale (top left) and as a graph $(x,y,f(x,y))$ (bottom left). Its M\"obius signature surface~(\ref{eq:sig}) is shown at right.}\label{fig:sig3D}\vspace{-1mm}
\end{figure}

\begin{figure}[t!]\centering
\includegraphics[width=130mm]{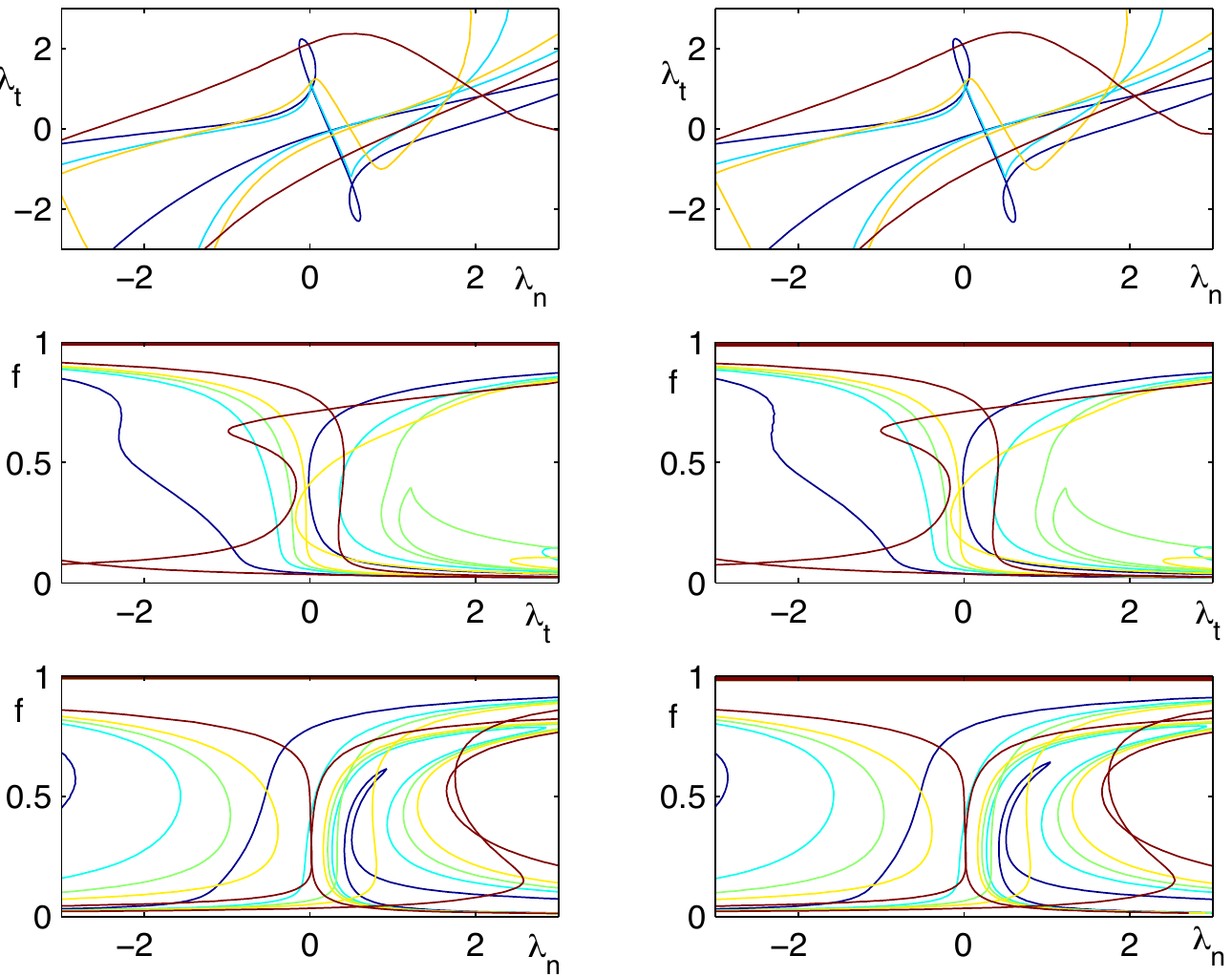}
\caption{The invariant signature $(f,\lambda_n,\lambda_t)$ shown for $f$ in the left column and for $f\circ\varphi^{-1}$ in the right column. Top: contours 0.2, 0.4, 0.6, and 0.8 of~$f$; middle and bottom: contours $-1$, $-0.25$, 0, 0.25, and~1 of~$\lambda_n$ (resp.~$\lambda_t$). The two invariants are almost identical in appearance (see, e.g., the $-1$ (dark blue) contour of $\lambda_t$ near $(\lambda_n,f)=(1,0.6)$, which is slightly dif\/ferent in the left and right columns).}\label{fig:sig}
\end{figure}

\begin{Example}\label{example3} \looseness=1 As a more numerical example, we take 9 similar blob-like functions, constructed as the sum of four random 2D Gaussian functions, and their M\"obius images under a random Mobius transform, and compare their invariant signatures. The functions and their Mobius-transformed variants $f\circ\varphi^{-1}$ are shown in Fig.~\ref{fig:9blobs1} as level set contours, while the invariant signatures are shown in Fig.~\ref{fig:9blobs3}. Because the whole invariant signature surfaces are very complicated, we show just the signature curve corresponding to the level set $f^{-1}(0.5)$. This depends only on the f\/irst 3 derivatives of~$f$ on the level set. Because~$\lambda_n$ and $\lambda_t$ take values in $[-\infty,\infty]$, we use coordinates $(\arctan(\lambda_t/4),\arctan(\lambda_n/4))$. Clearly, even this very limited portion of the signature serves to distinguish the M\"obius-related pairs extremely sensitively. In some cases, the invariants change extremely rapidly along the level set, so that even though they are eva\-luated accurately, the resulting contours of the M\"obius-related pairs do not overlap. This would need to be taken into account in the development of a distance measure on the invariant signatures.
\end{Example}

\begin{figure}[t!]\centering
\includegraphics[width=6cm]{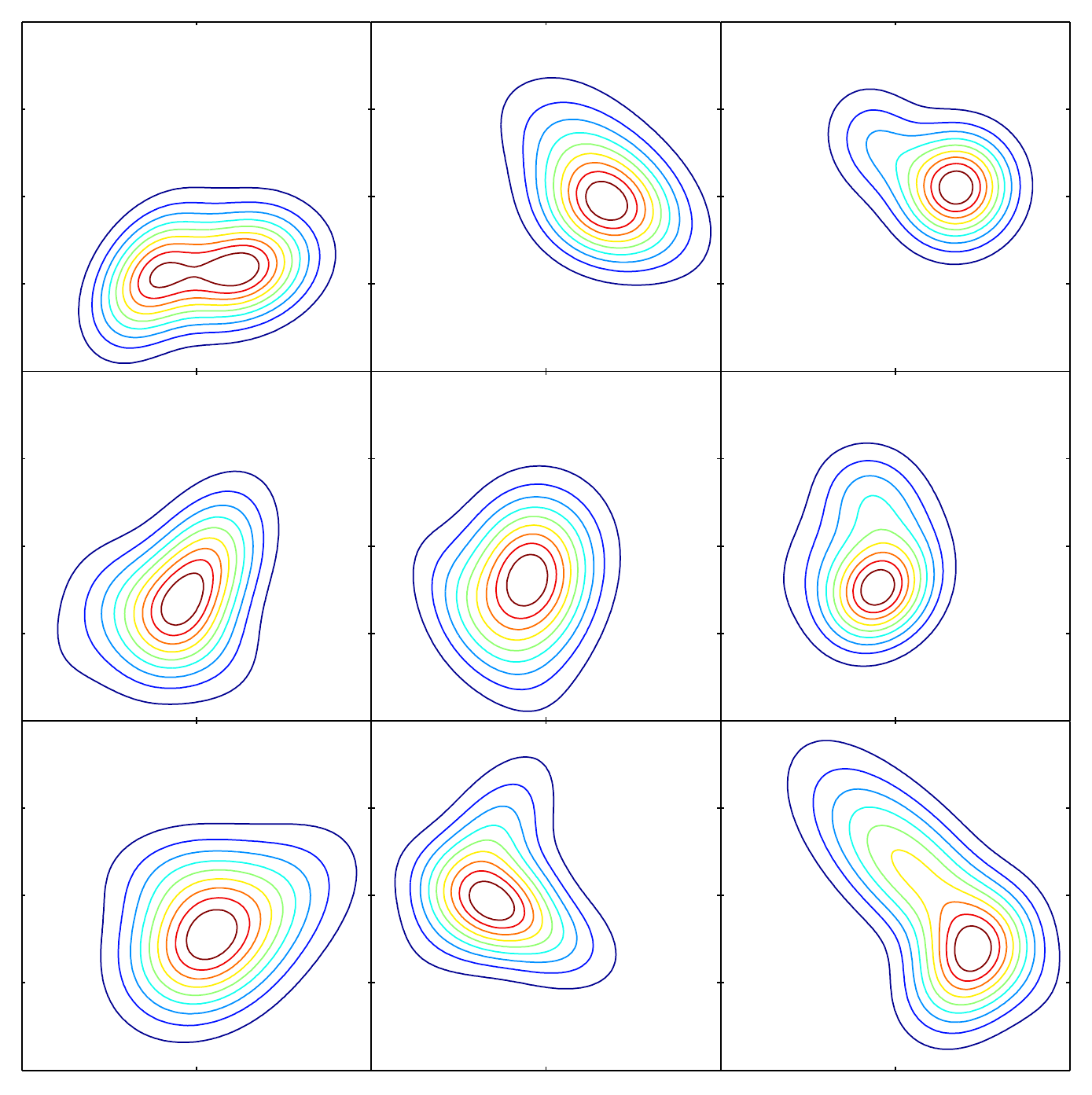} \quad
\includegraphics[width=6cm]{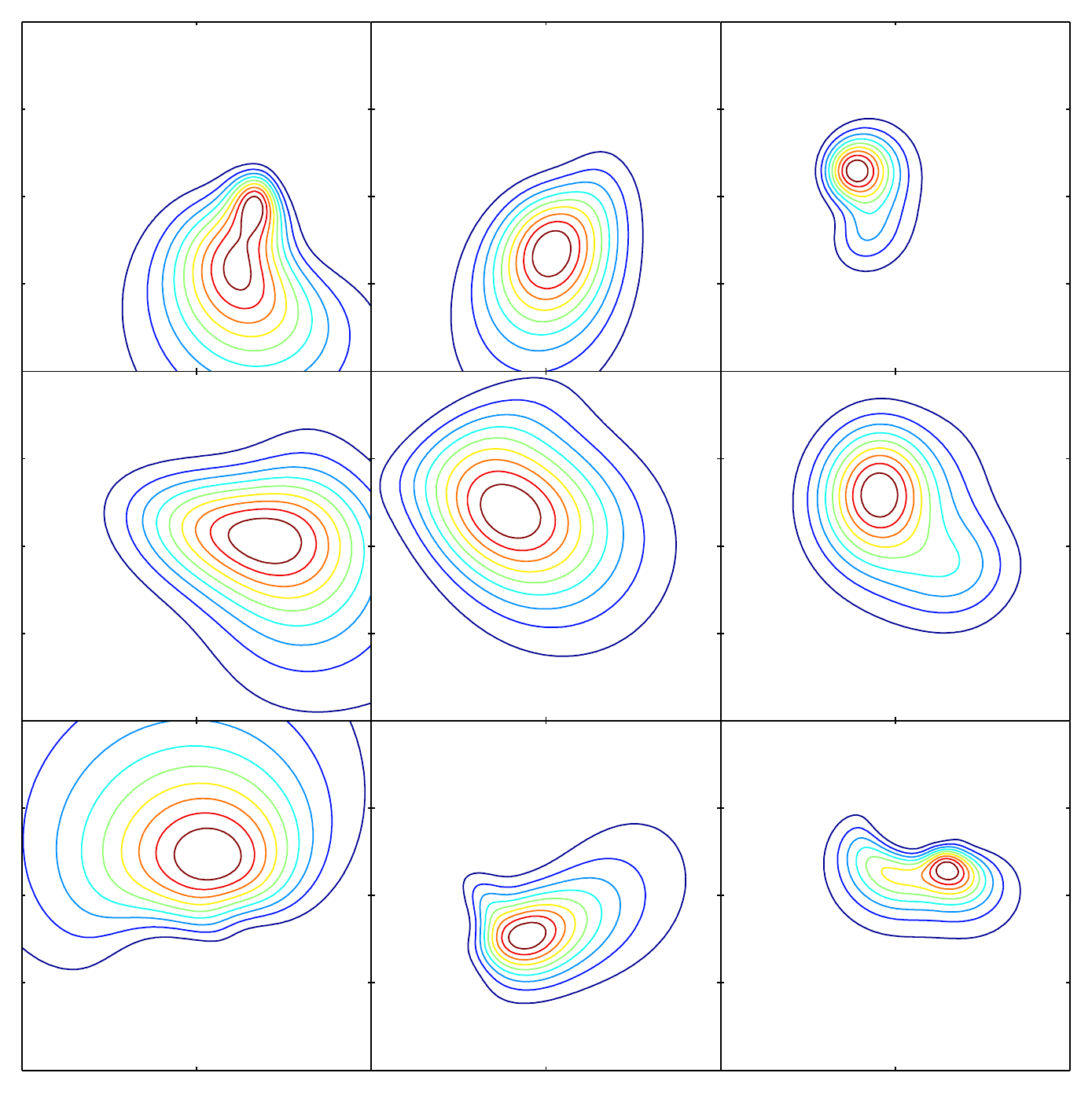}
\caption{Nine random blob-like functions are shown on the left. Each is given by the sum of 4 random Gaussians, with the range of the resulting function scaled to $[0,1]$. The domain is $[-1,1]^2$ and the functions are discretized with $h=1/80$ giving $161\times161$ grey-scale images. For each of the 9 functions~$f$, a~random M\"obius transformation $\varphi$ is chosen and the composition $f\circ\varphi^{-1}$ shown on the right, evaluated on the domain $[-1,1]^2$. The transformations have parameters $b=0$, $a$ uniform in an annulus with inner radius 0.7 and outer radius 1.3, $d=1$, and $c$ with uniform argument and normal random modulus with standard deviation 0.6. The contours 0.1, 0.2,\dots,0.9 of the functions are shown.}\label{fig:9blobs1}
\end{figure}

\begin{figure}[t!]\centering
\includegraphics[width=6cm]{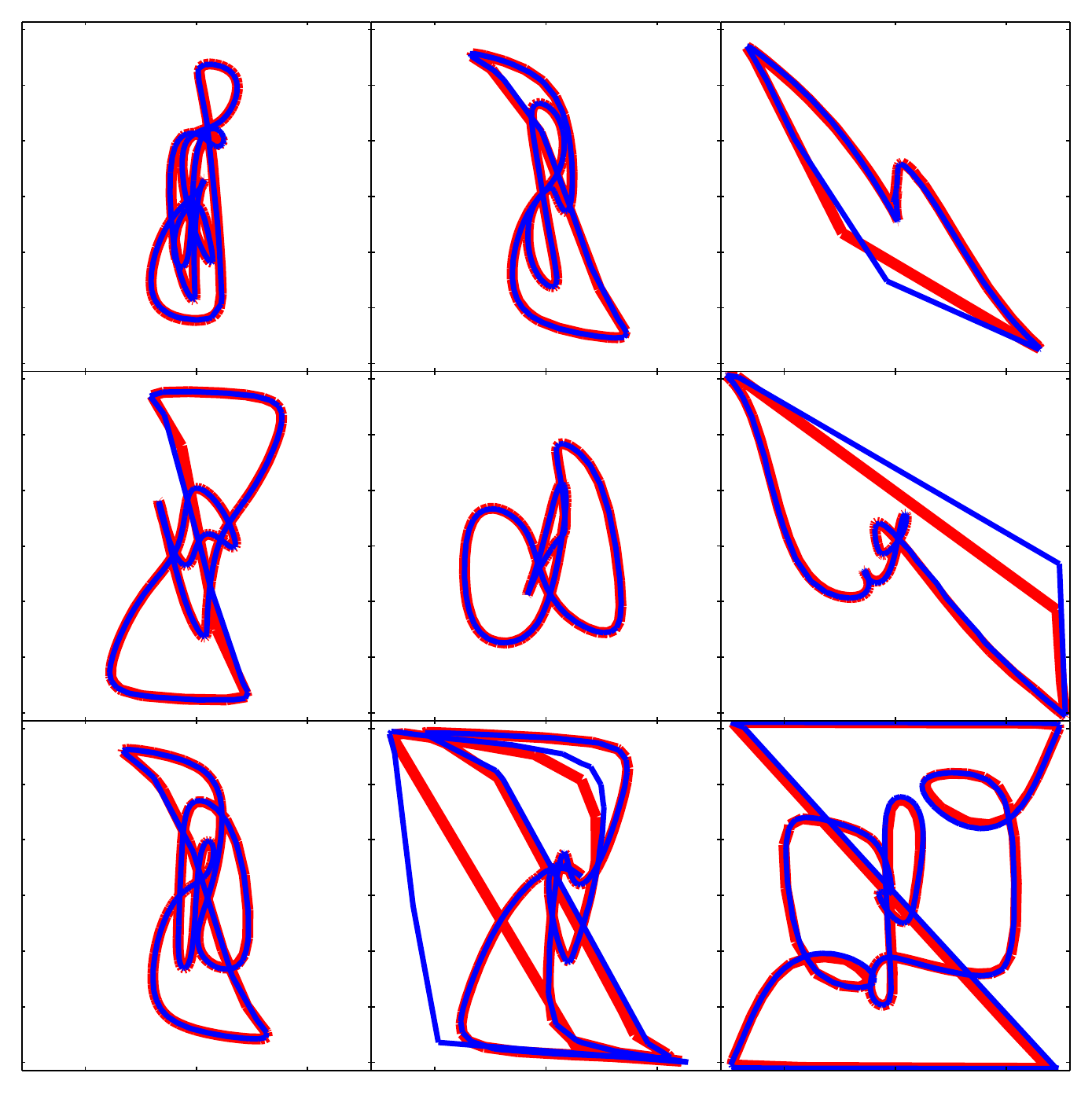}
\caption{The invariant signature $(\arctan(\lambda_t/4),\arctan(\lambda_n/4))$ evaluated on the level set $f^{-1}(0.5)$ is calculated by central dif\/ferences for each of the images in Fig.~\ref{fig:9blobs1} left (shown in blue) and for the corresponding images in Fig.~\ref{fig:9blobs1} right (shown in red). The domains are $[-\pi/2,\pi/2]^2$. The signature curves distinguish the M\"obius-related pairs very sensitively; only tiny f\/inite dif\/ference errors are visible. However, some errors related to insuf\/f\/icient resolution of the signature curves are clearly visible.}\label{fig:9blobs3}
\end{figure}

\begin{Example}In this example we illustrate the extreme sensitivity of the invariant signature by evaluating it on 9 very similar images, together with their M\"obius transformations. Each original image is a blob function generated as in Example~\ref{example3}, but with parameters varying only by~$\pm 5\%$. The M\"obius transformations have the form $1/(1 + c z)$ where $c$ is normally distributed with standard deviation~0.1. The 0.5-level contours of the original and transformed images are shown in Fig.~\ref{fig:10blobs1}, and their signatures in Fig.~\ref{fig:10blobs2}. The signature is extremely sensitive to tiny changes in the image, but not to M\"obius transformations.
\end{Example}

\begin{figure}[t!]\centering
\includegraphics[width=8cm]{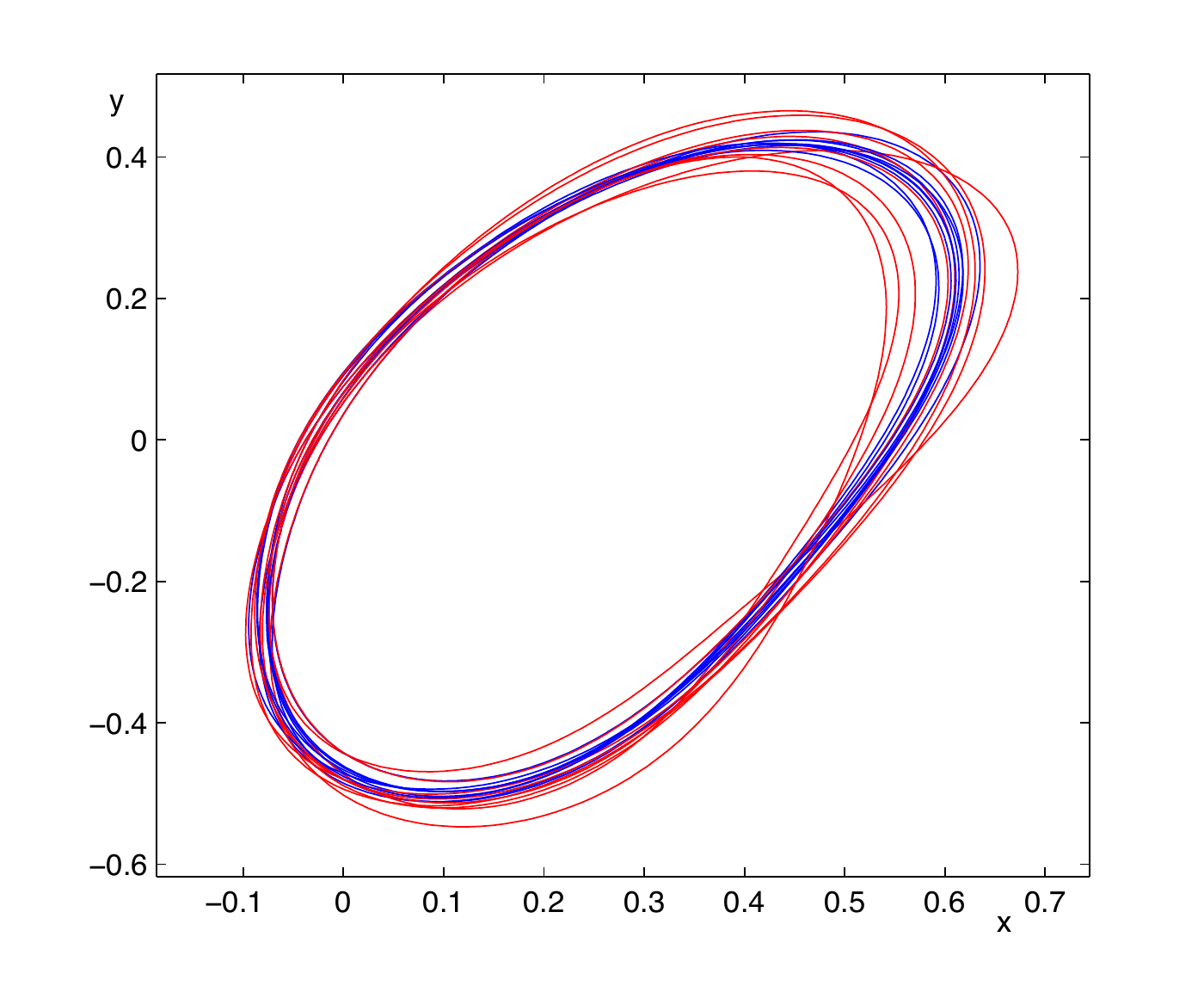}
\caption{The 0.5-level contour of 9 very similar blob-like images are shown in blue, and of their M\"obius transformations in red. Only the central $80\times 80$ portion of the $161\times 161$ images are shown.}\label{fig:10blobs1}
\end{figure}

\begin{figure}[t!]\centering
\includegraphics[width=8cm]{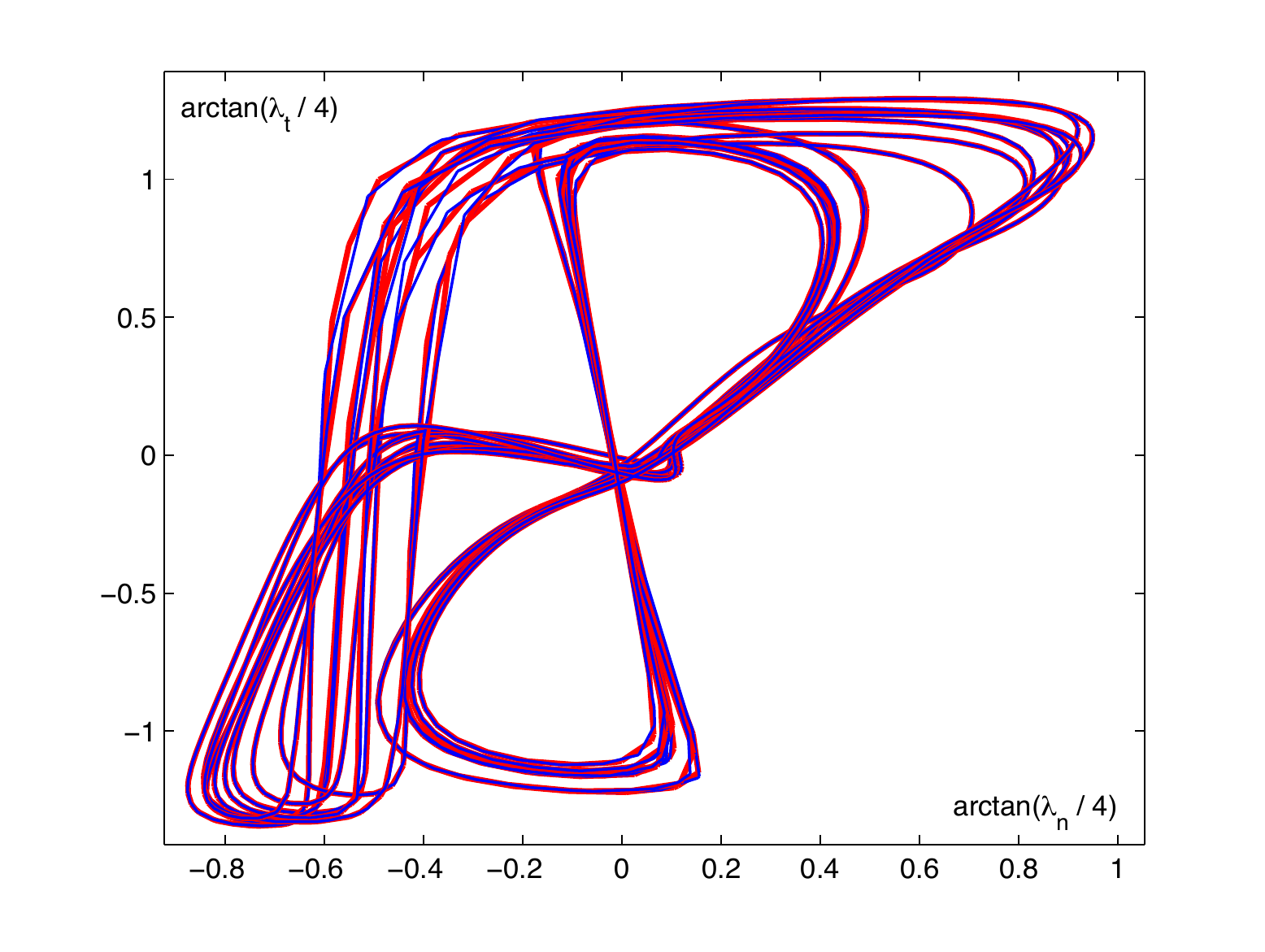}
\caption{The invariant signature $(\arctan(\lambda_t/4),\arctan(\lambda_n/4))$ evaluated on the level set $f^{-1}(0.5)$ is shown for each of the images in Fig.~\ref{fig:10blobs1} (blue) and for their M\"obius transformations (red).}\label{fig:10blobs2}
\end{figure}

We do not have a full understanding of the properties of this invariant signature with respect to the criteria listed in Section~\ref{sec:invariants}. It is certainly fast, small, local, and lacks redundancy and suppression. It has a good numerical approximation on smooth (or smoothed) images. Is it complete? That is, given an image, does its signature surface determine the image up to a M\"obius transformation? Suppose we are given a small piece of signature surface, parameterized by~$(u,v)$, say. We are given three functions $\tilde f(u,v)$, $\tilde \lambda_n(u,v)$, and $\tilde \lambda_t(u,v)$, and need to determine (by solving three PDEs) three functions $f(x,y)$ (the image), $x(u,v)$, and $y(u,v)$ (the coordinates). Typically, the solution of these PDEs will be determined by some boundary data. This suggests that distinct images with the same signature are parameterized by functions of~1 variable; a~kind of near completeness that may be good enough in practice.

Although very sensitive, the fact that it is not continuous at critical points means that it does not have good discrimination in the sense of Section \ref{sec:invariants}. (It falls into the `more false negatives' region of Fig.~\ref{fig:diagram}.) Near nondegenerate critical points, the signature blows up in a well-def\/ined way, so it is possible that there exists a~metric on signatures that leads to robustness and good discrimination.

\section{Conclusion}

In this paper we have developed M\"{o}bius invariants of both curves and images, and proposed computational methods to evaluate both, demonstrating them on a variety of examples. In Section \ref{sec:invariants} we identif\/ied a set of properties that are important for invariants, principally that there was a small set of invariants that were quick to compute, numerically stable, robust (so that noisy versions of the same curve have similar invariants) and yet suf\/f\/iciently discriminatory (so that dif\/ferent objects have dif\/ferent invariants).

While dif\/ferential invariants are not generally robust when dealing with noise, they of\/fer good discrimination and are cheap to compute; this leads us to the M\"{o}bius arc-length. The cross-ratio is more robust, but requires a large set of points to be evaluated, and blows up as the pairs of points approach each other. In order to make this reparameterization-invariant, we used a Fourier transform. This lead to a method of computing M\"{o}bius invariants that satisf\/ies the properties that we have outlined, as is demonstrated in the numerical experiments, see particularly Fig.~\ref{fig:distancecorr}.

For images, the extra information means that it is possible to compute a relatively simple three-dimensional signature based on the function value at each point together with two functions of the M\"{o}bius arclength, along and perpendicular to level sets of the image intensity. It is computationally cheap, extremely sensitive to non-M\"obius changes in the image, but insensitive to M\"obius transformations of the image.

\subsection*{Acknowledgements}
This research was supported by the Marsden Fund, and RM by a James Cook Research Fellowship, both administered by the Royal Society of New Zealand. SM would like to thank the Erwin Schr\"odinger International Institute for Mathematical Physics, Vienna, where some of this research was performed.

\pdfbookmark[1]{References}{ref}
\LastPageEnding

\end{document}